\documentclass[11pt]{article}

\usepackage[margin=1in]{geometry}

\usepackage{fullpage}
\usepackage{parskip}
\usepackage[colorlinks=true, citecolor=blue, linkcolor=red]{hyperref}

\usepackage{microtype}
\usepackage{graphicx}
\usepackage{subfigure}
\usepackage{booktabs} 
\usepackage{nopageno}

\usepackage[ruled]{algorithm2e} 

\SetAlFnt{\small}
\SetAlCapFnt{\small}
\SetAlCapNameFnt{\small}
\SetAlCapHSkip{0pt}
\IncMargin{-\parindent}

\usepackage{hyperref}

\usepackage{amsmath}
\usepackage{amssymb}
\usepackage{mathtools}
\usepackage{amsthm}
\usepackage{caption}

\usepackage[capitalize,noabbrev]{cleveref}

\usepackage{natbib}

\usepackage[ruled]{algorithm2e} 
\SetAlFnt{\small}
\SetAlCapFnt{\small}
\SetAlCapNameFnt{\small}
\SetAlCapHSkip{0pt}
\IncMargin{-\parindent}

\SetCommentSty{mycommfont}

\usepackage[utf8]{inputenc}

\usepackage{nicefrac}

\usepackage{xcolor}

\usepackage{amsfonts}
\usepackage{amsmath}
\usepackage{amssymb}
\usepackage{bbm}

\newcommand{\E}{\mathbb{E}}

\newcommand{\eps}{\varepsilon}
\newcommand{\hr}{\widehat{r}}
\newcommand{\hb}{\widehat{b}}
\newcommand{\hv}{\widehat{v}}
\renewcommand{\tt}{\widetilde{t}}
\newcommand{\tcalO}{\widetilde{\mathcal{O}}}
\newcommand{\calO}{\mathcal{O}}
\newcommand{\OPT}{\textsf{\upshape OPT}}

\newcommand{\tq}{\widetilde{q}}
\newcommand{\hDP}{\widehat{\textsf{\upshape DP}}}

\newcommand{\Bern}{\textsf{Bern}}
\newcommand{\Regret}{R}
\newcommand{\ALG}{\textsf{\upshape ALG}}
\newcommand{\REP}{\textsf{\upshape REP}}
\newcommand{\tH}{\widetilde{H}}
\newcommand{\xhdr}[1]{\vspace{1mm} \noindent{\bf #1}}
\newcommand\numberthis{\addtocounter{equation}{1}\tag{\theequation}}
\newcommand{\calE}{\mathcal{E}}
\newcommand{\calC}{\mathcal{C}}

\newcommand{\ist}{i^{\star}}
\newcommand{\jst}{j^\star}

\newcommand{\barb}{\bar{b}}
\newcommand{\bbar}{\bar{b}}
\newcommand{\calA}{\mathcal{A}}

\newcommand{\regEXT}{R_{\textsc{EXT}}}
\newcommand{\DES}{\textsc{DES}}
\newcommand{\regDES}{R_{\DES}}
\newcommand{\BDES}{\textsc{B-DES}}
\newcommand{\pist}{{\pi^{\star}}}
\newcommand{\calI}{\mathcal{I}}

\theoremstyle{plain}
\newtheorem{theorem}{Theorem}[section]
\newtheorem{proposition}[theorem]{Proposition}
\newtheorem{lemma}[theorem]{Lemma}
\newtheorem{corollary}[theorem]{Corollary}
\theoremstyle{definition}

\theoremstyle{remark}

\usepackage[textsize=tiny]{todonotes}

\title{Preferences Evolve And So Should Your Bandits: Bandits with Evolving States for Online Platforms}

\author{Khashayar Khosravi\thanks{\url{khashayar.khv@gmail.com}. Part of the work was done while the author was an AI resident in Google Research.} \and Renato Paes Leme\thanks{Google Research NYC, \url{renatoppl@google.com}} \and Chara Podimata\thanks{MIT, \url{podimata@mit.edu}. Part of the work was done while the author was an intern at Google Research.}\and Apostolis Tsorvantzis\thanks{National Technical University of Athens, \url{atsorvat@gmail.com}}}
\date{\today}

\begin{document}

\maketitle

\begin{abstract}
We propose a model for learning with bandit feedback while accounting for deterministically evolving and unobservable states that we call \emph{Bandits with Deterministically Evolving States} ($\BDES$). The workhorse applications of our model are learning for recommendation systems and learning for online ads. In both cases, the reward that the algorithm obtains at each round is a function of the short-term reward of the action chosen and how ``healthy'' the system is (i.e., as measured by its state). For example, in recommendation systems, the reward that the platform obtains from a user's engagement with a particular type of content depends not only on the inherent features of the specific content, but also on how the user's preferences have evolved as a result of interacting with other types of content on the platform. Our general model accounts for the different rate $\lambda \in [0,1]$ at which the state evolves (e.g., how fast a user's preferences shift as a result of previous content consumption) and encompasses standard multi-armed bandits as a special case. The goal of the algorithm is to minimize a notion of regret against the best-fixed \emph{sequence} of arms pulled, which is significantly harder to attain compared to standard benchmark of the best-fixed action in hindsight. We present online learning algorithms for any possible value of the evolution rate $\lambda$ and we show the robustness of our results to various model misspecifications. 
\end{abstract}

\pagestyle{plain}

\newcommand{\kst}{k^\star}

\section{Introduction}\label{sec:intro}




Online platforms serving ads and general recommendation systems have become an integral part of our everyday lives. Both for ads and recommendations, platforms strive for high engagement of the users with the content. Understanding better what drives user engagement has been a major research question since the advent of online advertising (and more recently, recommendation systems) not just because of its potential to drive revenue, but also, due to its potential to increase user satisfaction. Despite the proliferation of models put forth to explain user behavior, most of them have focused on users that are short-sighted/myopic; i.e., users who make engagement decisions not caring about their prior interactions with the platform.  

A landmark paper by~\citet*{hohnhold2015focusing} proposes a model of user behavior that accounts for \emph{evolving preference effects}\footnote{We use the terms ``evolving preference'' and ``long-term'' effects interchangeably.} and \emph{empirically evaluates} it in the context of the Google auction. They describe the phenomenon of \emph{ad-blindness} and \emph{ad-sightedness}, in which a user changes their inherent propensity to click on or interact with ads based on the quality of previously viewed ads. For example, click-baits may be more likely to generate a click now, but are also likely to decrease the user's happiness with the system and hence, click less often in the future (ad blindness). Instead, a high quality ad may lead to higher user engagement in the future (ad sightedness). The situation is similar when it comes to general recommendation systems, where researchers have been trying to capture the evolving state of user preferences as a result of their exposure to specific types of content (see e.g., \cite{kapoor2015just}).

Although~\citet{hohnhold2015focusing} posit a behavioral model for users and then optimize its parameters, in this paper, we study the problem from the \emph{theoretical} viewpoint of \emph{bandit optimization}. Specifically, we cast the problem of learning to recommend to users with evolving preferences as a bandit learning problem, where the choices made in each round have long-term impact on the user, and thus, on the platform's reward. Roughly speaking, in our model the reward collected by the algorithm in each round is affected by both the short term reward of the arm played and the state, which \emph{deterministically}\footnote{This is the main novelty of our model. There has been a lot of work in non-stationary bandits and we discuss the connections with our model in the Related Work below.} changes based on the platform's actions. For the online ads example, this means that the learner has to choose between which ads to show to a user when each ad has both an intrinsic clickability and an effect on the users propensity to click on future ads. Both are initially unknown to the learner, who can only observe clicks.

\subsection{Our Contributions} 

Our first contribution is to propose a bandit-based model for learning to choose a sequence of actions, which captures the long-term effects of prior decisions that we term \emph{Bandits with Deterministically Evolving States} ($\BDES$) (Section~\ref{sec:model-blindness}). As we discuss extensively in \cref{sec:model-blindness}, our model captures mathematically the behavioral observations of \citet{hohnhold2015focusing}. To the best of our knowledge, we are the first to propose such a model capturing ad blindness/sightedness in the context of bandit learning. Our model and results are stated for a more general setting, as they are useful for capturing other important applications of learning with long-term effects too, like evolution of preferences in recommendation systems.

We outline our model for learning in $\BDES$ as we find it important for highlighting our contributions, and defer the formal description to Section~\ref{sec:model-blindness}. There are $K$ arms. Each arm $i \in [K]$ is associated with a tuple $(r_i, b_i) \in [0,1]^2$, which is unknown to the learner. $r_i$ denotes the \emph{in-the-vacuum (IV) reward} of arm $i$, i.e., the reward sampled from this arm, if it were to be played in isolation, and abstracting away from the long-term effects of previously pulled arms. $b_i$ denotes the \emph{end state (ES)}\footnote{We use the ES $b_i$ to model precisely what Hohnhold et al. call the ``long-term impact''. Quoting from their paper: ``The long-term impact is what would happen if the experiment launched and users received the experiment treatment in perpetuity --- in other words, it is the impact in the limit $t \to \infty$.''}  of this arm if one were to play it for an infinite number of rounds as a result of the long-term effects. Crucial to our model is the notion of a \emph{``state''}\footnote{We use the wording ``state'' to match similar literature in MAB. In reality, $q_t$ is a multiplier.} $q_t \in [0,1]$, which captures the effects of the sequence of actions played so far to the reward that the learner obtains at each round $t \in [T]$. The state transition function is governed by a \emph{known}\footnote{We assume that $\lambda$ is known as the platforms can estimate (through market research; see \cref{subsec:hohnholdetal}) the speed/rate at which the system transitions after each round. The quantities they are missing are the rewards. For our algorithms, we only need to know the general ``region'' where $\lambda$ belongs.} general state evolution parameter $\lambda \in [0,1]$ as follows: $q_{t+1} = (1 - \lambda) q_t + \lambda b_{I_t}$, where $I_t$ is the arm played at round $t$. The state is \emph{never} observed by the learner. Instead, when the learner chooses arm $i$ to play at round $t$, they only observe reward $\tilde r_{i, t} \sim \Bern( q_t \cdot r_i )$. We refer to $\tilde r_{i,t}$ as the \emph{state-augmented reward}. We adopt the perspective of the platform and wish to design algorithms that minimize a notion of \emph{regret} $\regDES(T)$, i.e., the cumulative difference between the loss of the algorithm and the loss of an optimal, \emph{benchmark policy} in hindsight. Note that this regret definition is \emph{strictly harder} to minimize compared to \emph{external} regret, which only compares against the best-fixed \emph{action} in hindsight. We show that standard no-external regret algorithms can have \emph{linear} regret against our harder benchmark.


Next, we provide online learning algorithms for \emph{any} value of $\lambda \in [0,1]$ (see Table~\ref{table:summary} for the full picture). To give the reader intuition about our results, we start with the middle case, where $\lambda$ is neither too big, nor too small (Section~\ref{sec:gen-evol-rate}). Our algorithm (Algorithm~\ref{algo:ETC-known-iR}) first builds estimates about $\hr_i$ and $\hb_i$ and subsequently, when these estimates are such that $|\hr_i - r_i| \leq \eps$ and $|\hb_i - b_i| \leq \eps$ it ``feeds'' them as input in a Dynamic Program (DP) algorithm designed to compute the offline optimum sequence of arms, if $\{(r_i, b_i)\}_i$ were known in advance. We show that by mis-estimating $(r_i, b_i)$ by a factor of $\eps$, the DP algorithm can obtain reward at least $(1-\eps)\OPT$, where $\OPT$ is the optimal expected reward for an instance of $\BDES$. The regret bound obtained for this case is $\calO ( K^{1/3} T^{2/3} \log(\lambda) / \log (1 - \lambda))$.

The key technical point in our approach for Section~\ref{sec:gen-evol-rate} is that although we want to disentangle the learning of $r_i$ and $b_i$, the learner only observes \emph{state-augmented} rewards, and the state is \emph{never} revealed to the learner. We circumvent this by observing that because of the form of the state transition function for our problem, playing repeatedly the arm with the highest ES restores the state to approximately $1 - \eps$. This means that at the next round, we are able to obtain \emph{almost} a clean sample for $r_i$, despite observing a state-augmented reward!

Interpreting the $\calO ( K^{1/3} T^{2/3} \log(\lambda) / \log (1 - \lambda))$ regret bound obtained for this algorithm, we note that it provides vacuous guarantees (i.e., linear regret) for ``extreme'' values of $\lambda$ (i.e., $\lambda \to 0$ or $\lambda \to 1$). This is because for $\lambda \to 0$ Algorithm~\ref{algo:ETC-known-iR} needs to spend linear in $T$ rounds in order to build good estimators $\hr_i, \hb_i$. To address this, we design different algorithms for small and large values of $\lambda$. Specifically, in Section~\ref{sec:small-lambda}, we address the case where $\lambda \in [0, \widetilde{\Theta}(1 / T)]$. 

\begin{table*}[t]
\small
\addtolength{\tabcolsep}{-1pt}
\centering
\makebox[0pt]{
\begin{tabular}{c|c|c|c|c}
\toprule
{} & $\lambda \in [0, \widetilde{\Theta}(1/T^2)]$ & $\lambda = \widetilde{\Theta}(T^{-a/b})$, for $b < a <2b$ & $\lambda \in (\Theta(1/T), \widetilde{\Theta}(1 - 1/\sqrt{T}))$ & $\lambda \in [\widetilde{\Theta}(1 - 1/\sqrt{T}), 1]$\\
\hline
$\regDES(T)$ & $\tcalO(\sqrt{KT})$ (Thm~\ref{thm:small-lambda})& $\calO(T^{b/a})$ (Thm~\ref{thm:small-lambda}) & $\tcalO(K^{1/3} T^{2/3})$ (Thm~\ref{thm:regret-ETC}) & $\tcalO(K \sqrt{T})$ (Thm~\ref{thm:regret-sticky})\\
\bottomrule
\end{tabular}
}
\caption{Summary of regret rates proved. $\tcalO(\cdot)$ hides terms poly-logarithmic in $K, T, \lambda$.}
\label{table:summary}
\end{table*}

To address the case of small $\lambda$'s, we treat the states as \emph{exogenously} given quantities (i.e., not influenced by the choices of the algorithm in previous rounds) that affect the realized rewards per-round and apply the standard EXP3.P algorithm (see Section~\ref{sec:small-lambda} for a discussion on the choice of EXP3.P). The technical difficulty here is that EXP3.P provides only \emph{external} regret guarantees, so when one wants to translate the guarantees to $\DES$ regret, they need to more carefully handle the error picked up by EXP3.P as a result of not comparing with the optimal policy as a benchmark. 

In Section~\ref{sec:sticky-arms}, we study the case where $\lambda \in [\widetilde{\Theta}(1 - 1/\sqrt{T}), 1]$. We call this the ``sticky arms'' case, since for $\lambda = 1$ once the learner plays an arm $I_t$ at round $t$, the state becomes $q_{t+1} = b_{I_t}$. Through a careful application of the re-arrangement inequality, we show that when $\lambda = 1$, the optimal sequence of actions is periodic with a cycle of at most $2$ arms; hence, an algorithm can define meta-arms consisting of pairs of arms $(i \diamond j)$ and play a bandit learning algorithm on the meta-arms instead. To avoid picking up linear-in-$T$ regret or scaling inefficiently with the number of arms in this case, we need to be careful in the way we alternate playing different meta-arms. We do so by coupling the arms in \emph{batches} and alternatively playing them without discarding any reward samples.

In \cref{sec:robustness}, we study the robustness of our results to model misspecifications. Specifically, we consider two types of misspecifications; first, that the state-augmented reward is \emph{not deterministically} affected by the state $q_t$, but there is also some added $\sigma$-subGaussian noise; and second, that $\lambda$ is fully unknown. For the first model misspecification, we show that all our algorithms are fully agnostic to the $\sigma$-subGaussian noise and their performance deteriorates only by an extra $\sigma T$ factor for all $\lambda$. For the second model misspecification, we show that under an assumption on the discrepancy of the arms' rewards or an assumption on the region where $\lambda$ it is possible to obtain sublinear regret. 

We conclude with a discussion of open questions and directions in \cref{sec:limitations-blindness}.

\subsection{Related Work}


Closest to our work is the work of \citet{hohnhold2015focusing}, who also studied models of evolving preferences but focused solely on a model suited to ad blindness/ad sightedness. Our work has orthogonal strengths. \citet{hohnhold2015focusing} first estimate the ad blindness/sightedness parameters and then they use these to redesign online ad auctions. We, instead, study a more fundamental learning setting, our results are not calibrated to a single search engine, and our algorithms cover other settings with evolving preference effects as well (e.g., recommendation systems).

From the online learning literature, our work has connections with papers both on Multi-Armed Bandit (MAB) problems and more general RL settings. There has been a lot of recent interest in settings where the expected rewards of the arms evolve over time (i.e., there is a long-term effect on the system). \citet{RottingBandits} and \citet{RBnoharderthanstoch} study ``rotting bandits'', where the long-term effect is that as you pull an arm the realized reward presented to the learner decreases. The main difference with our problem is that in ``rotting bandits'' there is no way to ``replenish'' what you lost from an arm as you kept pulling it. Additionally, the benchmark policy in rotting bandits is to greedily play the optimal arm at each round, had you known everything in advance, which is not at all the case in our setting. 

\citet{kleinbergimmorlicarechargingbandits} study ``recharging bandits'', where rewards accrue as time goes by since the last time the arm was played. In ``blocking bandits'' \citep{basu2019blocking,ContextualBlockingBandits,AdvBlockingBandits} playing an arm makes it unavailable for a fixed number of time slots thereafter. In \citet{heidari2016tight,leqi2021rebounding}, the rewards of the arms increase/decrease as they get played. In ``rested bandits'' \citep{gittins1979bandit} an arm's expected rewards change only when it is played. In ``restless bandits'' \citep{whittle1988restless} rewards evolve independently from the play of each arm. In \citep{cella2020stochastic} the rewards increase as a function of the time elapsed since the last pull. In ``recovering bandits'' \citep{pike2019recovering} the expected reward of an arm is expressed as a function of the time since the last pull drawn from a Gaussian Process with known kernel. In \cite{warlop2018fighting}, the rewards are linear functions of the recent history of actions. In \citep{mintz2020nonstationary}, rewards are a function of a context that evolve according to known deterministic dynamics. 
In our case, the \emph{inherent} rewards of the arms do not change; instead, they are filtered through the state which is affected by all previously played arms.

\citet{MABAdvScaling20} consider the case where the arms have a stochastic component and an adversarial one, which is chosen at each round by adversary. The final mean reward is the product between the stochastic and adversarial components. The difference with our setting is that in our case, stochastic reward is multiplied by the \emph{state}, which is defined deterministically based on the sequence of prior actions, and cannot be chosen arbitrarily by an adversary. In a similar vein, \citet{CorrArms} consider the setting where the rewards of pulling different arms are correlated. In our case, rewards are also correlated but they are governed by the state. Correlations also really arise once you pull the arms sequentially, as opposed to their problem, where correlation \emph{requires} arms to be pulled simultaneously

Our work is also related to RL with MDPs with deterministic transition functions (e.g., \citep{ortner2008online, ADMDP} for stochastic and adversarial respectively) and with~\citep{ortner2012online} which studies a stochastic RL setting with a continuous state space. The core difference with our work, however, is that the aforementioned works assume that the learner can observe the state that they find themselves in at each round.

\section{Model \& Preliminaries}\label{sec:model-blindness}



We introduce the setting of \emph{Bandits with Deterministically Evolving States} ($\BDES$). Each arm $i \in [K]$ is associated with tuple $(r_i, b_i) \in [0,1]^2$. $r_i$ denotes the \emph{in-the-vacuum (IV) reward} of arm $i$, i.e., the reward sampled from this arm, if it were to be played in isolation, and abstracting away from the long-term effects of previously pulled arms. $b_i$ denotes the \emph{end state (ES)} of this arm if one were to play it for an infinite number of rounds as a result of the long-term effects. Let $I_t$ denote the arm chosen at round $t$, and $H_{s:t}^{\ALG}$ the history of arms played by algorithm $\ALG$ from round $s$ until round $t$, i.e., $H_{s:t}^{\ALG} = \{I_\tau\}_{\tau = s}^t$. To capture the evolution of preferences as a result of the arms played by $\ALG$ so far, we use the notion of a ``state'', denoted by $q_t(H_{1:t-1}^\ALG)$. Formally, we assume that when playing arms according to $\ALG$, the state evolves deterministically as:
\begin{align*}
q_{t+1}\left(H_{1:t}^{\ALG}\right) &= q_t \left( H_{1:t-1}^\ALG\right) + \lambda \cdot \left(b_{I_t} - q_t\left( H_{1: t-1}^\ALG \right)\right) = (1-\lambda) \cdot  q_t \left( H_{1:t-1}^\ALG \right) + \lambda \cdot b_{I_t}, \numberthis{\label{eq:qt-def}}
\end{align*}
where $\lambda$ is a \emph{known} \emph{evolution rate} controlling how much the present state is affected by the most recently pulled arm versus the earlier arms. Eq.~\eqref{eq:qt-def} models that we take gradient steps on the state function parametrized by arm $I_t$ (Fig.~\ref{fig:state-evolution}). We use $q_0$ for the initial state, and assume that $q_0 = 1$ without loss of generality. When clear from context, we drop the dependence of $q_t(\cdot)$ on the history.
\begin{figure}[t!]
    \centering
    \includegraphics[width=0.33\textwidth]{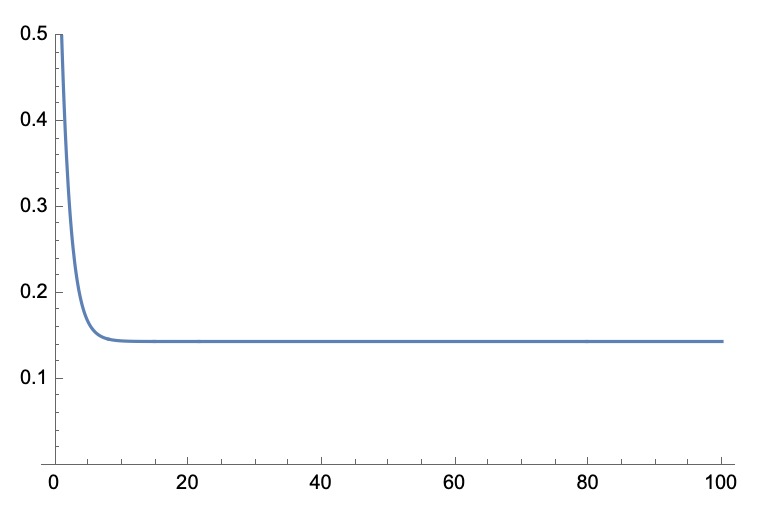}
    \caption{State evolution function for a fixed arm with $b_i = 0.15$ and $\lambda = 0.5$.}
    \label{fig:state-evolution}
\end{figure}

At each round $t$, the learning protocol is: First, the learner pulls arm $I_t \in [K]$. Second, they observe the \emph{state augmented} reward $\tilde{r}_{I_t,t}$ that is sampled from $\Bern(r_{I_t} \cdot q_t(H_{1:t-1}^\ALG))$. Third, the state is updated as in Eq.~\eqref{eq:qt-def}. Importantly, however, the learner never gets to observe the current state $q_t(H_{1:t-1}^\ALG)$ and they also never observe the tuple $(r_{I_t}, b_{I_t})$. The learner's goal is to choose a sequence of arms $\{I_t\}_{t \in [T]}$ that minimize a notion of \emph{regret} that accounts for states. Let $\pi^{\star}: [K] \to [K]^T$ denote the \emph{policy} choosing the sequence of arms to maximize the expected reward when the tuples $(r_i, b_i)_{i \in [K]}$ are known (i.e., $\pi_t^{\star}$ corresponds to the arm chosen at round $t$ by the optimal policy). The \emph{deterministically-evolving-state} ($\DES$) regret is defined as: 
\begin{align*}
        \Regret_{\DES}(T) = \E &\left[\sum_{t \in [T]} r_{\pi_t^{\star}} \cdot q_t \left( H_{1:t-1}^{\pi^{\star}} \right) - \sum_{t \in [T]} r_{I_t} \cdot q_t \left(H_{1:t-1}^\ALG\right) \right]
\end{align*}
For simplicity, in the remainder of the paper we use $q_t^{\pi^{\star}} = q_t(H_{1:t-1}^{\pi^{\star}})$. The $\DES$ regret is similar to \emph{policy regret} \cite{PolicyRegretFirst} (rather than the \emph{external regret}), where the benchmark accounts for long-term effects (i.e., is not just a static fixed arm play).

Before we move to the technical sections of the paper, we find it useful to translate the general model to our two motivating examples of online ads and recommendation systems. 

\xhdr{Translation to the ads example.} For online ads, the \emph{arms} correspond to \emph{ads} and $\pi^{\star}$ corresponds to the optimal ad schedule. The \emph{state} of round $t$ corresponds to the user's propensity to click after engaging with the system for $t$ ads. The evolution rate $\lambda$ corresponds to the speed according to which \emph{ad sightedness/blindness} affects the user's satisfaction from round to round. The \emph{IV reward} of an arm corresponds to the \emph{inherent click-through-rate (CTR)} that the ad would have for a given user had there not been long-term effects. The \emph{end state} of an arm corresponds to the baseline sightedness/blindness of the respective arm, had it been presented infinitely. The fact that at each round $t$ the reward is sampled from $\Bern(r_{I_t} \cdot q_t(H_{1:t-1}^\ALG))$ translates to observing a click with probability $r_{I_t} \cdot q_t(H_{1:t-1}^\ALG)$. The modeling choice that the rewards are state-augmented captures the fact that the probability that a user clicks on an ad depends not only on the ad's IV reward but also on the overall happiness of the user interacting with the system (i.e., the state).

\xhdr{Translation to the recommendation systems' example.} For a recommendation system, the \emph{arms} are different types of \emph{content} and $\pi^{\star}$ corresponds to the optimal schedule for exposing users to said content. The \emph{state} of round $t$ corresponds to the user's happiness with interacting with the system after engaging with it and consuming content for $t$ rounds. The evolution rate $\lambda$ corresponds to the rate according to which the user's preferences are shaped as a result of how they interact with the system and their original preferences for each piece of content they are exposed to. The IV reward of an arm corresponds to the \emph{inherent utility} that the content would have for a given user had there been no evolving preference shaping effects. The ES of an arm corresponds to the baseline utility of the respective content for the user, had it been presented to them infinitely. The fact that at each round $t$ the reward is sampled from $\Bern(r_{I_t} \cdot q_t(H_{1:t-1}^\ALG))$ translates to observing an engagement (e.g., likes, comments) with probability $r_{I_t} \cdot q_t(H_{1:t-1}^\ALG)$. This is because the probability that a user engages with a piece of content depends not only on the content's IV rewards but also on the overall happiness of the user interacting with the system (i.e., the state).

\subsection{Experimental evidence for the functional form in $\BDES$}\label{subsec:hohnholdetal}

The functional form of the state evolution in $\BDES$ (Equation~\eqref{eq:qt-def}) is based on the functional form derived from experiments in the qualitative study of \citet{hohnhold2015focusing}. \citet{hohnhold2015focusing} conducted ``ad blindness experiments'' where they select a random subset of users and exposes them to a different mix of ads and measure how the CTR of those users evolves over time as compared to the control group. They plotted the data collected on the CTR evolution~\citep[Fig. 2]{hohnhold2015focusing} and the best fit function was \citep[Eq. (4)]{hohnhold2015focusing}: $\tilde{U}(t) = \alpha' (1 - e^{-\beta t})$, where $\tilde{U}(t)$ denotes the \emph{change} in CTR associated with a specific user for a set of ads at round $t$ from round $0$ and $\alpha', \beta$ are parameters that we are going to specify shortly. 

If we translate the experiment setup to our model, it would be as if we expose a user to an arm / ad with $(r, b)$ repeatedly. From Eq.~\eqref{eq:qt-def}, the state (inherently tied with the CTR) then changes as: 
\begin{center}
$q_{t+1} = (1- \lambda)^t q_0 + b \sum_{s=0}^t \lambda (1-\lambda)^{1-t-s} b = b - (1- \lambda)^t (b - q_0)$
\end{center}
Since $\tilde{U}(t)$ corresponds to the change in CTR in $t$ rounds, then in the language of our model: $\tilde{U}(t) = q_t - q_0 = (b - q_0) ( 1- (1- \lambda)^t )$. In other words, comparing our state evolution function with $\tilde{U}(t)$, they have exactly the same functional form with $\alpha' = b - q_0$ and $\beta = - \log (1 - \lambda)$. Because of the connection between $\beta$ and $\lambda$, one can use \citet{hohnhold2015focusing}'s methods for estimating $\lambda$. 

Note that \citet{hohnhold2015focusing} do not study a bandit/online problem and algorithm. Their goal is to demonstrate the ad blindness effect and study methodologies to estimate it from experiments. We take their insights and apply to online decision making. 


\subsection{External vs DES Regret}

Achieving sublinear $\DES$ regret is significantly harder compared to achieving sublinear external regret. In fact, in general we need completely new algorithms to achieve sublinear $\DES$.

\begin{proposition}\label{lem:external-vs-DES}
    Let algorithm $\ALG$ be a no-external regret algorithm (e.g., \textsc{UCB, AAE, EXP3} etc). For any such algorithm $\ALG$, there exists a family of instances $\calI$ for which $\regDES(T) = \Omega(T)$.
\end{proposition}

The proof can be found in Appendix~\ref{app:model}. At a high level, we show an instance for $\lambda = 1$ where any $\ALG$ with sublinear external regret converges to one particular arm except for $o(T)$ rounds, but the optimal sequence for $R_\DES(T)$ involves strictly more than one arm. The explanation of why we can guarantee that we know the optimal sequence in this case comes later in this paper (\cref{sec:sticky-arms}).
\begin{figure}[t!]
    \centering
    \includegraphics[width=0.75\linewidth]{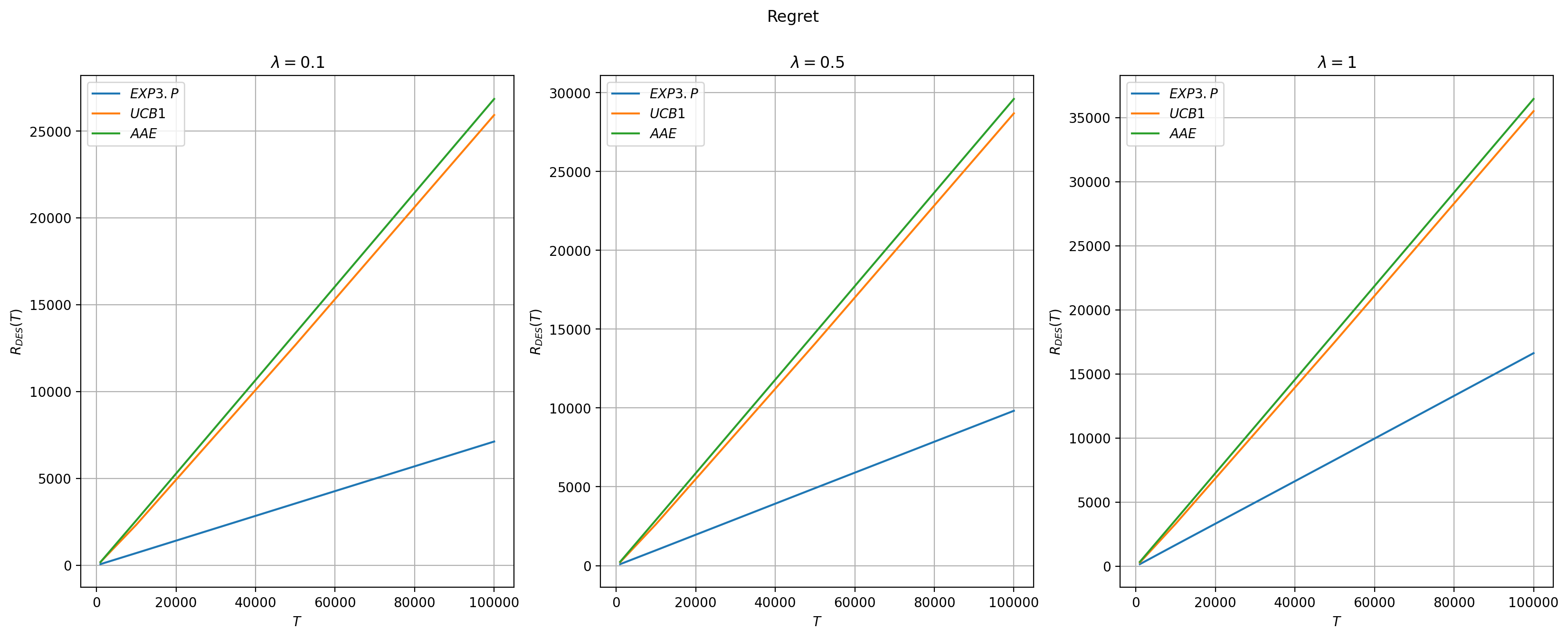}
    \caption{External vs DES regret for an instance with $2$ arms and varying $\lambda$.}
    \label{fig:external-vs-DES}
\end{figure}
To verify that standard algorithms fail even when $\lambda < 1$, we also ran experiments on simulated data for a carefully crafted instance of only $2$ arms comparing 3 well-known algorithms for minimizing regret (Fig. \ref{fig:external-vs-DES}). In all cases, the DES regret that the standard algorithms incur is linear in $T$.

\section{General Evolution Rate Algorithm}\label{sec:gen-evol-rate}

In this section, we present the algorithm for a general $\lambda$. Formally, we prove the following theorem.

\begin{theorem}\label{thm:regret-ETC}
For $\lambda \in (0,1)$, tuning $\delta = \eps/4$, $M = \log (T) / \eps^2$ and $\eps = \left( \frac{K \cdot \log (T) \cdot \log (\lambda)}{T \cdot \log (1 - \lambda)} \right)^{1/3}$, Algorithm~\ref{algo:ETC-known-iR} incurs regret $\regDES(T) = \calO\left( \left(\frac{K \log (T) \log (\lambda)}{\log (1 - \lambda)}\right)^{1/3} \cdot T^{2/3}\right)$.
\end{theorem}

We first present a relaxation for the problem of learning in $\BDES$, where for each arm the learner has \emph{estimates of bounded error} about $(r_i, b_i)$. Then, we show the efficiency (compared to $\pi^{\star}$) of a Dynamic Programming (DP) approach that takes as input these estimates and outputs a sequence of arms. To design the DP algorithm, it is useful to have a closed-form solution for the state at each $t$. The missing proofs of this section can be found in \cref{app:DP-relax}.

\subsection{Relaxation: Dynamic Programming with Approximate Rewards}\label{sec:DP-relax}

A useful first step in building the DP algorithm is computing the closed-form solution of the states that are induced by an algorithm. The proof of the lemma is done through induction.

\begin{lemma}\label{lem:state-closed-form}
Let $\ALG$ be an algorithm pulling arm $I_t$ at round $t$. The closed form solution for computing the state at each round is: 
\begin{equation}\label{eq:qt-closed-form}
q_{t+1} \left( H_{1:t}^\ALG \right) = (1-\lambda)^t + \lambda \cdot \sum_{s = 0}^{t-1} (1 - \lambda)^{t-1-s} \cdot b_{I_s}
\end{equation}
\end{lemma}

This closed-form solution for the state is important, since it allows us to directly decompose $q_{t+1}(H_{1:t}^\ALG)$ to the ES of the arms pulled so far. As a result, as we argue below, we do not need to have \emph{full} knowledge of the exact rewards of the arms; instead, good approximations are enough to give us a solution that is close to the optimal solution $\OPT := \sum_{t \in [T]} q_t^{\pi^{\star}} \cdot r_{\pi^\star_t}$. 

\begin{lemma}\label{lem:approx-DP}
Let $\hDP$ denote the expected reward of the solution returned by a dynamic programming algorithm with inputs $(\hr_i, \hb_i)_{i \in [K]}$, where $|\hr_i - r_i| \leq \delta$ and $|\hb_i - b_i| \leq \delta$. 
Then, $\hDP \geq \OPT - \delta T$.
\end{lemma}

\xhdr{Computational complexity of the DP algorithm.} The offline problem of finding $\pist$ through the DP algorithm has a knapsack-style structure; to see this, for each possible horizon $T$, associate for each sequence of played arms $(i_1, \dots, i_t)$ for $t \leq T$ the total expected reward obtained by $\rho_t$ and the final state $q_t$. This means that for each $t$, we have a list of possible $(\rho_t, q_t)$. Let us call this list $F_t = \{(\rho_t, q_t)\}_t$. Then, one can construct $F_{t+1}$ from $F_t$ as $F_{t+1} = \{(\rho + \hr_j \cdot q, (1-\lambda) \cdot q + \lambda \cdot \hb_j)\}$ for tuple $(\rho, q) \in F_t$ and $j \in [K]$. Similar to the dynamic programming in knapsack, $F_t$ can grow exponentially large in $T$. Indeed, in every round $|F_t| \leq |F_{t-1}| \cdot K$ and $F_0 = 1$ which leads to an algorithm with exponential complexity in $T$: $\calO \left( K^T \right)$.

Because of the knapsack-style structure, we can create an FPTAS for the problem (see e.g.,~\cite{williamson2011design}) as in Algorithm~\ref{algo:fptasdp}. To see this, one can round to multiples of $\epsilon$, remove ``dominated'' tuple components, and end up with at most $t/\epsilon$ points in each $F_t$. Formally, a pair $(\rho, q)$ dominates another pair $(\rho', q')$ if $\rho \geq \rho'$ \emph{and} $q \geq q'$.

\begin{algorithm}[t!]
    \caption{FPTAS DP for $\BDES$}\label{algo:fptasdp}
    \DontPrintSemicolon
    \LinesNumbered
    \SetAlgoNoEnd
    \textbf{Input.} Estimators $\hat{r}_i, \hat{b}_i$ and approximation parameter $\epsilon$. \;
    Tuple sequence initialization: $F_0  \gets  \{ (0, 1) \} $ \tcp*{$\rho_0 = 0, q_0 = 1$}
    \For {$t \in [T]$}{
        $F_t \gets F_{t-1}$ \tcp*{Start from the previously optimal sequence}
        \For {all tuples $(\rho, q) \in F_{t-1}$ }{
            \For{$ i \in [K]$}{
                Update $F_t \gets F_{t} + \{ (\rho + \left \lfloor \frac{1}{\epsilon} \hat{r}_i \cdot q \right \rfloor, (1- \lambda) q + \lambda \hat{b}_i) \}$
            }   
        }
        Remove dominated pairs from $F_t$.
    }
    Return the sequence of arms associated with $F_T$, denoted $S$.
    \end{algorithm}

    \begin{lemma}\label{lem:FPTAS}
    For any $\epsilon > 0$, \Cref{algo:fptasdp} is an FPTAS with runtime $\calO(K T^2 / \epsilon)$ for $\hDP$ when given estimators $(\hr_i, \hb_i)_{i \in [K]}$. 
    \end{lemma}
The proof of the lemma is based on the fact that since $S$ is the sequence chosen from the algorithm and that the optimal expected reward in the $\epsilon$-discretized setting is at most $\epsilon \OPT$ away from the optimal expected reward. For the runtime, note that at each round, the algorithm has to check $T / \epsilon$ tuples. In the remainder, we abstract away from the use of the FPTAS algorithm by tuning $\epsilon = 1/T$. As it will become clear, this only adds an additional $\calO(1)$ term to the regret of our algorithms.


\subsection{Estimating the IV Rewards and ES} \label{subsec:gen-lam}

For the ease of exposition, we describe the results of this part with a simplifying assumption; namely, that there exists a known ``replenishing'' arm $i_R$ for which it holds that $b_{i_R} \in [1-\eps, 1]$. As will be made clear later, $\eps$ is a parameter that the learner can tune and it trades off with the regret that the algorithm picks up. At the end of the section, we explain how the general case (without the replenishing arm assumption) can be analyzed, and we defer the formal algorithm and analysis to Appendix~\ref{app:gen-replenish-arm}. The full proofs of this section can be found in Appendix~\ref{app:gen-lambda}.

At a high level, it may seem impossible to disentangle learning the $r_i$'s and $b_i$'s just by observing the reward realization, which depends on their product. In fact, this hardness persists even if one of the two (either the $r_i$'s or the $b_i$'s) was known. To see this, we present two simple thought experiments.

For the first one, assume that the $b_i$'s were known. Due to Lemma~\ref{lem:state-closed-form}, this would then translate to us knowing the state at which we are at any round. In that case, we could simply build estimators $\hr_i$ for the $r_i$'s such that $|\hr_i - r_i| \leq \delta$ with high probability. Given the $\hr_i$'s and the actual $b_i$'s we could then feed $(\hr_i, b_i)_{i \in [K]}$ to the DP and obtain a solution that is $\delta T$ close to the $\OPT$ (Lemma~\ref{lem:approx-DP} ). Tuning $\delta$ appropriately would then give us a no-regret algorithm. The challenge is that in reality the $b_i$'s are also unknown and we cannot understand the state where the system is at any point.

For the second thought experiment, assume that $r_i$'s are now known, but the $b_i$'s are not. Similarly to before, we could now build estimators $\hb_i$ that are $\delta$-close to $b_i$, and then use again the DP solution. Again, we cannot really use this solution as-is, since both $(r_i, b_i)_{i \in [K]}$ are not known. 

Our setting, however, possesses a key property which allows us to disentangle the learning of $r_i$'s and $b_i$'s. The property is related to the deterministic way in which the state transitions and is the following: irrespective of the history of plays and the starting state, playing repeatedly the same arm $i$ for a fixed number of $N$ rounds makes the state become approximately equal to $b_i$. Moreover, $N$ is constant with respect to $\eps$ and $T$. The proof of the lemma can be found in Appendix~\ref{app:gen-lambda}.

\begin{lemma}\label{lem:state-approx-bi}
Fix an arm $i \in [K]$ and a scalar $\eps > 0$. Assume that at some round $s$, after a history of plays $H'$, we are at state $q_s$. Then, playing repeatedly arm $i$ for $N (\lambda) \leq c(\lambda) \cdot \log (1 / (\lambda \eps))$ rounds (where $c (\lambda) = \log^{-1}(1/(1-\lambda))$) makes the state become $q_{N(\lambda)}$, such that: $| q_{N(\lambda)} - b_i| \leq \eps$.
\end{lemma}

An important corollary is that irrespective of the history of plays and the current state, if one were to play the replenishing arm for $N_R:= N(\lambda)$ rounds, then, the state returns (approximately) to $q_0$.

\begin{corollary}\label{lem:replenish-time}
Let $q_s$ be the state reached at some round $s$ after history of plays $H'$. Then, playing repeatedly $i_R$ for $N_R \leq c(\lambda) \cdot \log \left( 1 / (\lambda \eps)\right)$ times (where $c(\lambda) = \log^{-1} (1/(1-\lambda))$) makes the state become $q_{N_R} = 1 - \eps + \lambda \eps > 1 - \eps$.
\end{corollary}

Given Lemma~\ref{lem:state-approx-bi} and Corollary~\ref{lem:replenish-time}, we now have a way to disentangle the learning of the $r_i$'s and the $b_i$'s. This is precisely the idea behind our algorithm: first, build estimators for the $r_i$'s and subsequently, use these when trying to infer the $b_i$'s. The tricky part arises because of the fact that the $r_i$'s and the $b_i$'s are connected multiplicatively. Note that for Lemmas~\ref{lem:rew-estimator1},~\ref{lem:hv-estimator} and~\ref{lem:hb-estimator} that follow, we use a fixed $\eps$. We tune this $\eps$ optimally in the end to obtain the no-regret guarantee.

\begin{algorithm}[t!]
\caption{$\BDES$ general $\lambda$, known $i_R$}\label{algo:ETC-known-iR}
\DontPrintSemicolon
\LinesNumbered
\SetAlgoNoEnd
Set $\eps, \delta, M$ as stated in Theorem~\ref{thm:regret-ETC}.\;
Initialize rounds $t=1$. \;
\tcc{Build estimators $\{\hr_i\}_{i \in [K]}$}
\For{arm $i \in [K]$}{
Initialize reward estimate $\hr_i = 0$. \;
\For{blocks $j \in [M]$}{ 
\For(\tcp*[h]{State $\geq 1-\eps$}){pulls $1, \dots, N_R$}{ 
Play arm $i_R$. \;
Update $t \gets t+1$.
}
Play arm $i$, observe reward $R_j^i$, and update: $\hr_i \gets \hr_i + \nicefrac{R_j^i}{M}$. \tcp*{$q\approx1-\eps$}
Update $t \gets t+1$.
}
}
\tcc{Build estimators $\{\hb_i\}_{i \in [K]}$}
\For{arm $i \in [K]$}{
Initialize state estimator $\hv_i = 0$. \;
\For{pulls $1, \dots, N(\lambda)$}{
Play arm $i$. \;
Update $t \gets t+1$. \;
}
\For{blocks $j \in [M]$}{
Play arm $i$, observe reward $S_j^i$, and update: $\hv_i \gets \hv_i + \nicefrac{S_j^i}{M}$. \tcp*{Play $i$ when $q \approx b_i$}
}
Compute baseline reward estimator: $\hb_i = \nicefrac{\hv_i}{\hr_i}$.\;
}
Play arm $i_R$ for $N_R$ rounds, updating $t \gets t+1$ after each one.\tcp*{State $\geq 1 - \eps$}
Feed $( \hr_i, \hb_i )$ in the Dynamic Programming algorithm and play the solution until the end of horizon $T$.
\end{algorithm}%

\xhdr{Notation.} To simplify the exposition and the notation with the explicit dependence on the history of plays, we denote with $t_{j}^i$ the round $t$ after the final play of arm $i_R$ during block $j$ for arm $i$ (i.e., Ln 9), and with $\tt_{j}^i$ the round $t$ after the final play of arm $i$ during block $j$ for arm $i$ (i.e., Ln 17).

We first prove that the reward estimators we build are good approximations for the true rewards.

\begin{lemma}\label{lem:rew-estimator1}
For the IV reward estimator of Algorithm~\ref{algo:ETC-known-iR} of each arm $i$ and any scalar $\delta > 0$, it holds that: $\Pr \left[ \left| \hr_i - r_i \right| \geq \delta \right] \leq 2 \exp \left( -2 M \cdot (\delta-\eps)^2 \right)$.
\end{lemma}

For arm $i \in [K]$, let $v_i = r_i \cdot b_i$. Then, we denote as $\hv_i$ the estimator of $v_i$ for each arm $i$ through Algorithm~\ref{algo:ETC-known-iR}. We show that $\hv_i$ is a good estimator for $v_i$ for all $i \in [K]$.

\begin{lemma}\label{lem:hv-estimator}
For estimator $\hv_i$ of Algorithm~\ref{algo:ETC-known-iR} for arm $i$ and any $\delta > 0$, it holds that: $\Pr \left[\left|\hv_i - v_i \right| \geq \delta \right] \leq 2 \exp \left( -2 M \cdot (\delta - \eps)^2 \right)$.
\end{lemma}

The proofs of Lemmas~\ref{lem:rew-estimator1} and~\ref{lem:hv-estimator} are based on an application of Hoeffding's inequality combined with Corollary~\ref{lem:replenish-time} to control the time it takes for the state to return to almost $1$.

It remains to show that using the estimators $\hr_i, \hv_i$, one can obtain a good estimator for the end states $\hb_i$ for each arm $i \in [K]$. This is trickier than showing that estimators $\hr_i, \hv_i$ are individually good proxies for the true $r_i, b_i$; the hardness comes from the fact that $\hr_i, \hv_i$ are \emph{almost} unbiased estimators and we are dealing with their product. 

\begin{lemma}\label{lem:hb-estimator}
For the end state estimators of Algorithm~\ref{algo:ETC-known-iR} for each arm $i$ and any scalar $\delta > 0$, it holds that: $\Pr [ | b_i - \hb_i | \geq \delta ] \leq 4\exp ( - 2M \cdot (\eps^2 - \eps \delta )) + 4 \exp ( -2 M \cdot (\eps - \delta)^2 )$.
\end{lemma}

\begin{proof}
Fix an arm $i \in [K]$ and let us use $e_v$ and $e_r$ to denote the following quantities: $e_v = \hv_i - v_i$ and $e_r = \hr_i - r_i$ respectively. Then, we have that:
\begin{align*}
   &\Pr \left[ \left| \frac{\hat{v}_i}{\hat{r}_i} - \frac{v_i}{r_i} \right| \geq \delta \right] =\Pr \left[ \left| \frac{v_i+e_v}{r_i+e_r} - \frac{v_i}{r_i} \right| \geq \delta \right] =\Pr \left[ \left| \frac{r_i e_v - v_i e_r}{r_i(r_i+e_r)} \right| \geq \delta \right] \\
   &\leq \Pr \left[ \left| \frac{e_v}{r_i+e_r} \right| + \left| b_i \frac{e_r}{r_i + e_r} \right|\geq \delta \right] &\tag{triangle ineq.}\\ 
   &\leq \underbrace{\Pr \left[ \left| \frac{e_v}{r_i+e_r} \right| \geq \delta/2 \right]}_{Q_1} + \underbrace{\Pr \left[ b_i \cdot \left| \frac{e_r}{r_i+e_r} \right| \geq \delta/2 \right]}_{Q_2} \numberthis{\label{eq:hoeff-bhat}}
\end{align*}
To upper bound $Q_1$ and $Q_2$, we condition on the following event: $\calE_i' = \{ |e_r| \leq \delta \}$. Note that the probability with which the complement $\calE_i$ happens is given by Lemma~\ref{lem:rew-estimator1}: 
\begin{equation}\label{eq:calE-cond}
    \Pr [\calE_i] \geq 2 \exp \left(-2 M \cdot (\delta - \eps)^2 \right) 
\end{equation}
Rewriting $Q_1$:
\begin{align*}
    Q_1 &= \Pr \left[ |e_v| \geq \frac{\delta}{2} \cdot |r_i + e_r|\right]  \leq \Pr \left[ |e_v| \geq \frac{\delta}{2} \cdot \Big||r_i| - |e_r| \Big|\right] \numberthis{\label{eq:Q1-relax}}
\end{align*}
Conditioning on $\calE_i'$ we get: 
\begin{align*}
    \Pr \left[ |e_v| \geq \frac{\delta}{2}  \Big | |r_i| - |e_r| \Big | \Big | \calE_i' \right] &\leq \Pr \left[ |e_v| \geq \frac{\delta}{2} | r_i - \delta | \right] \leq \Pr \left[ |e_v| \geq \frac{\delta}{2} | r_i - \delta | \Big| \calE_i' \right] \\
    &= \Pr \left[ |e_v| \geq \frac{\delta}{2} | r_i - \delta | \right] \\
    &\leq 2 \exp \left( -2M \cdot \left(\frac{\delta}{2} \cdot |r_i - \delta| - \eps \right)^2 \right) &\tag{Lem.~\ref{lem:hv-estimator}}\\
    &\leq 2 \exp \left( -2 M (\eps^2 - \eps \delta) \right) \numberthis{\label{eq:Q1-cond}}
\end{align*}
where the third derivation is due to the fact that $\calE_i'$ depends on $e_r$ and none of the quantities that we take the conditional on depend on it too. Additionally, the last inequality is due to the fact that $|r_i - \delta| \leq 1$. From the law of total probability:
\begin{align*}
    Q_1 &= \Pr \left[ |e_v| \geq \frac{\delta}{2} \cdot |r_i + e_r| \, \Big | \, \calE_i'\right] \cdot \Pr \left[ \calE_i' \right]  + \Pr \left[ |e_v| \geq \frac{\delta}{2} \cdot |r_i + e_r| \, \Big | \, \calE_i\right] \cdot \Pr \left[ \calE_i \right] &\tag{Eq.~\eqref{eq:Q1-relax}} \\
    &\leq \Pr \left[ |e_v| \geq \frac{\delta}{2} \cdot \Big | |r_i| - |e_r| \Big | \, \Big | \, \calE_i'\right] \cdot \Pr \left[ \calE_i' \right] + \Pr \left[ |e_v| \geq \frac{\delta}{2} \cdot \Big | |r_i| - |e_r| \Big |\, \Big | \, \calE_i\right] \cdot \Pr \left[ \calE_i \right] \\
    &\leq 2 \exp \left( M \cdot \left( \eps^2 - \delta \right) \right) \cdot 1 + 1 \cdot 2 \exp \left( - 2M \cdot (\delta - \eps )^2 \right) 
\end{align*}
where the last inequality is due to Eqs.~\eqref{eq:calE-cond},~\eqref{eq:Q1-cond}. Next, we turn our attention to $Q_2$: 
\begin{align*}
    Q_2 = \Pr \left[ |e_r| \geq \frac{\delta}{2} \cdot \frac{|r_i + e_r|}{b_i}\right] \leq \Pr \left[ |e_r| \geq \frac{\delta}{2} \cdot \left| r_i + e_r \right| \right]
\end{align*}
where the inequality is due to the fact that $b_i \leq 1$. Using exactly the same reasoning as above, but now coupled with Lemma~\ref{lem:rew-estimator1} instead of Lemma~\ref{lem:hv-estimator} we have that: 
\begin{center}
$Q_2 \leq 2 \exp \left( M \cdot \left( \eps^2 - \eps \delta \right) \right)+ 2 \exp \left( - 2M \cdot (\delta - \eps )^2 \right)$
\end{center}
Adding the two upper bounds from $Q_1$ and $Q_2$ to Equation~\eqref{eq:hoeff-bhat} we get the stated result.
\end{proof}

Using Lemmas~\ref{lem:rew-estimator1} and~\ref{lem:hb-estimator} we can prove Theorem~\ref{thm:regret-ETC} by bounding the number of rounds it takes for the estimators to converge to approximately correct values with high probability and the regret picked up in the event of failing to converge. 

\xhdr{Sketch for the unknown $i_R$ case.} Note that two instances of arms $\{(r_i, b_i)\}_{i \in [K]}$ and $\{(c r_i, b_i / c)\}_{i \in [K]}$ for a scalar $c>0$ are equivalent. So, we can always scale the $b_i$'s appropriately to make sure that we have a ``replenishing'' arm. The next part is to show how Algorithm~\ref{algo:ETC-known-iR} changes if we do not know which among the $K$ arms is the replenishing arm. The only thing that changes is the way that we estimate $\hr_i$'s; instead of using $i_R$ as the benchmark arm, we sample randomly an arm $z$. Then, after enough rounds, we can guarantee that with high probability $\hat r_i \to \bar b r_i$, where $\bar b = \sum_j \nicefrac{b_j}{K}$. The second part of Algorithm~\ref{algo:ETC-known-iR} remains the same. We can then guarantee that we have obtained estimates $\hat b_i \to \nicefrac{b_i}{\bar b}$ with high probability. Tuning again $\delta, \eps, M$ we obtain the \emph{same order} regret guarantee. The details and the new algorithm can be found in Appendix~\ref{app:gen-replenish-arm}.

\xhdr{Interpreting the regret bound for extreme values of $\lambda$.} The regret bound of \Cref{algo:ETC-known-iR} is parametrized by $\lambda$; sublinear regret is only attainable when $\log(\lambda)/\log(1 - \lambda) < o(T^{1/3})$. For very small values of $\lambda$ (i.e., $\lambda \to 0$) \Cref{algo:ETC-known-iR} incurs \emph{linear} regret. Intuitively, this is because the smaller the $\lambda$, the more samples the algorithm needs to optimally estimate the $r_i$'s and the $b_i$'s. When $\lambda$ is close to $0$, the algorithm has to spend linear in $T$ rounds to estimate the $r_i$'s and the $b_i$'s. In Section~\ref{sec:small-lambda}, we present a different algorithm to overcome this issue for $\lambda \in [0, \widetilde{\Theta}(1/T)]$. On the other extreme, for $\lambda$ near $1$, the regret bound of \Cref{algo:ETC-known-iR} becomes vacuous. For $\lambda$ close to $1$ the state changes really fast; in fact, for $\lambda = 1$ at each round $t$ it becomes equal to the previous arm's end state (i.e., $q_t = b_{I_{t-1}}$). In \cref{sec:sticky-arms}, we use this property to design algorithms with sublinear regret for $\lambda \in [\widetilde{\Theta}(1 - 1/\sqrt{T}), 1]$.

\section{\emph{Slow} State Evolution: $\lambda \in [0, \widetilde{\Theta}(1/T)]$}\label{sec:small-lambda}

In this section, we study the case where the evolution rate $\lambda$ is \emph{small}, i.e., $\lambda \in [0, \widetilde{\Theta}(1/T)]$. Formally, we prove the following theorem.

\begin{theorem}\label{thm:small-lambda}
For $\lambda \in [0, \widetilde{\Theta}(1/T))$, Algorithm~\ref{algo:exp3.p} incurs sublinear regret.
\end{theorem}

For $\lambda = 0$, the problem becomes an instance of the standard stochastic $K$-MAB, since the state is always $q_t = q_0 = 1, \forall t$. So, applying the standard UCB algorithm~\citep{auer2002finite} guarantees regret $\regDES(T) = \calO (\sqrt{/TK \log T})$. For the remainder of the section, we discuss the case where $\lambda \in (0, \widetilde{\Theta}(1/T)]$. Roughly, we show that applying EXP3.P~\citep{auer2002nonstochastic} ``pretending'' that there are no states (i.e., taking $\E [\tilde r_{i,t}]$ as being exogenously decided; see Algorithm~\ref{algo:exp3.p} for formal description) incurs regret that is \emph{comparable} to $\regDES(T)$ up to some factors that we formalize below.

\begin{lemma}\label{lem:small-lambdas}
    For $\lambda \in (0, \widetilde{\Theta}(1/T)]$, EXP3.P incurs regret $\regDES(T) = \calO(\sqrt{KT \log K}) + (1 - (1 - \lambda)^T) \cdot \OPT$.
\end{lemma}

\begin{proof}
    Let $\{ I_t \}_{t \in [T]}$ be the sequence of arms played by EXP3.P (\Cref{algo:exp3.p}) and $\{\tilde q_t\}_{t \in [T]}$ the sequence of induced states as a result. Then, since EXP3.P minimizes the (expected) external regret:  
\begin{equation}\label{eq:ext-1}
\regEXT(T) = \sum_{t \in [T]} \tilde r_{I^*,t} - \sum_{t \in [T]} \tilde r_{I_t, t} \leq \calO \left( \sqrt{KT \log K} \right)
\end{equation}
where $I^* = \arg \max_{i \in [K]} \sum_{t \in [T]} \tilde r_{i,t} = \arg\max_{i \in [K]} \sum_{t \in [T]} \tilde q_t \cdot r_i$. EXP3.P treats the induced sequence of states $\tilde q_t$ as \emph{exogenously} given, and hence: $I^* = \sum_{t \in [T]} \tilde q_t \cdot \arg \max_{i \in [K]} r_i = \arg \max_{i \in [K]} r_i := \ist$. Using this derivation, Equation~\eqref{eq:ext-1} becomes: 
\begin{equation}\label{eq:ext-2}
    \regEXT(T) = \sum_{t \in [T]} \tilde q_t \cdot r_{\ist} - \sum_{t \in [T]} \tilde r_{I_t, t} \leq \calO \left( \sqrt{ KT \log K} \right)
\end{equation}
In the LHS of the above, we add and subtract the benchmark reward for $\regDES(T)$ (i.e., $\sum_{t \in [T]} q_t^{\pi^{\star}} r_{\pi^{\star}_t}$), so Equation~\eqref{eq:ext-2} becomes:
\begin{equation*}
    \regDES(T) + \underbrace{\left[\sum_{t \in [T]} \tilde q_t \cdot r_{\ist} - \sum_{t \in [T]} q_t^{\pi^{\star}} r_{\pi^{\star}_t}\right]}_{A} \leq \calO \left( \sqrt{ KT \log K} \right)
\end{equation*}
We next lower bound quantity $A$ as follows: 
\begin{align*}
    A &= \sum_{t \in [T]} \tilde q_t \cdot r_{\ist} - \sum_{t \in [T]} q_t^{\pi^{\star}} r_{\pi^{\star}_t} \geq \sum_{t \in [T]} (1 - \lambda)^t r_{\ist} - \sum_{t \in [T]} q_t^{\pi^{\star}} r_{\pi^{\star}_t} \numberthis{\label{eq:A-lb}}
\end{align*}
where the last inequality is because $q_t \geq (1-\lambda)q_{t-1}, \forall t$. Consider now \emph{any} sequence of arms that an algorithm could have played $\{J_t\}_{t \in [T]}$  and $\{\hat q_t\}_{t \in [T]}$ the associated induced states. Then: 
\begin{equation}\label{eq:other-seq}
\sum_{t \in [T]} r_{\ist} \geq \sum_{t \in [T]} r_{J_t} \geq \sum_{t \in [T]} \hat q_t r_{J_t} 
\end{equation}
where the first inequality is because $r_{\ist} \geq r_i, \forall i \in [K]$ and the second one because $1 \geq \hat q_t$. Since $\{J_t\}_{t \in [T]}$ is \emph{any} sequence, Equation~\eqref{eq:other-seq} must also hold for the \emph{optimal} sequence $\pist$. In other words, $\sum_{t \in [T]} r_{\ist} \geq \sum_{t \in [T]} q_t^{\pi^{\star}} r_{\pi^{\star}_t}$. Using the latter to relax the RHS of Equation~\eqref{eq:A-lb} we get:
\begin{align*}
    A &\geq \sum_{t \in [T]} q_t^{\pi^{\star}} r_{\pi^{\star}_t} \cdot (1 - \lambda)^t - \sum_{t \in [T]} q_t^{\pi^{\star}} r_{\pi^{\star}_t} \geq \left ( (1 - \lambda)^T -1 \right ) \sum_{t \in [T]} q_t^{\pi^{\star}} r_{\pi^{\star}_t}. 
\end{align*}
Putting everything together concludes the proof.
\end{proof}

\xhdr{Regret bound interpretation.} We distinguish two cases: $\lambda \in (0, \widetilde{\Theta}(1/T))$ and $\lambda = \widetilde{\Theta}(1/T)$. For the first case, $\lambda$ can be written as $\lambda = \widetilde{O}(T^{-a/b})$, with $a > b > 0$. Then, in the limit $T \to \infty$, we have that $(1 - (1-\lambda)^T) \cdot \OPT$ approaches $0$. Thus, for $\lambda \in (0, \widetilde{\Theta}(1/T))$ the regret incurred is sublinear 
, while for $\lambda = \widetilde{\Theta}(1/T)$, the algorithm obtains a $(1 - 1/e)$-\emph{approximate} regret guarantee. This means that for $\lambda = \widetilde{\Theta}(T^{-a/b})$ where $a \geq 2b$ the regret is $\regDES(T) = \calO(\sqrt{KT \log K})$, otherwise $\regDES(T) = \calO(T^{b/a})$.


\xhdr{EXP3.P versus UCB.} Why do we \emph{need} to use EXP3~\citep{auer2002nonstochastic} instead of UCB~\citep{auer2002finite}? In our setting, at each round $t$, the rewards $\tilde r_{i,t} \sim \Bern(q_t r_i), \forall i \in [K]$ are \emph{not} stochastic (and not even oblivious!). Indeed, recall that the state $q_t$ is \emph{not} exogenously given; instead, it is endogenously affected by the choices of arms played until round $t$. A version of UCB with enlarged confidence intervals by $\lambda T$ could have also worked, but note that EXP3.P has the added advantage of being fully agnostic to the exact $\lambda$.

\section{\emph{Fast} State Evolution: $\lambda \in [\widetilde{\Theta}(1 - 1/\sqrt{T}), 1]$}\label{sec:sticky-arms}


In this section, we study the case where $\lambda$ is close to $1$, specifically $\lambda \in [\widetilde{\Theta}(1 - 1/\sqrt{T}), 1]$. For this case, we show that we can obtain regret that is much lower compared to Section~\ref{sec:gen-evol-rate}. We call this special case of the problem the case of ``sticky'' arms. This is because when $\lambda = 1$, after playing arm $I_t$ then the current state becomes equal to arm $I_t$'s end state: $q_{t+1} = b_{I_t}$ (see Equation~\ref{eq:qt-def}). Formally, in this Section, we prove the following statement.

\begin{theorem}\label{thm:regret-sticky}
For $\lambda \in [\widetilde{\Theta}(1 - 1/\sqrt{T}), 1]$, Algorithm~\ref{algo:batched-bandit-meta-arms} incurs regret $\regDES(T) = \tcalO ( K \sqrt{T} )$. This regret bound is \emph{tight} (up to logarithmic factors) for $\lambda = 1$.
\end{theorem}


\subsection{``Sticky'' Arms: Evolution Rate $\lambda = 1$}

The fact that in ``sticky arms'' the state becomes the end state of the previously pulled arm has an important consequence: the optimal sequence of actions always alternates between $2$ arms.

\begin{lemma}\label{lem:cycle-2}
For $\lambda = 1$, the optimal sequence of actions is a \emph{cycle} of size $2$. 
\end{lemma}

\begin{proof}
The proof proceeds in $2$ steps. First, we prove that the optimal sequence $\pist$ contains \emph{minimum cycles} of length $N \leq K$ (a minimum cycle contains each arm only once) and then, that the best length is $2$. From the Pigeonhole Principle, there exists an arm $i$ that is played at least $2$ times in $K+1$ rounds (and $K + 1 < T$). Let $i$ be the arm that repeats in $\pist$ with the smallest length $N$ between the two rounds where it is repeated: in other words, if arm $i$ is played at round $t$ and then again in round $t + N$, there is no other arm $j \in [K] \setminus \{i\}$ that is played twice between rounds $t$ and $t+N$. We call the sequence of arms played between rounds $[t, t+N]$ \emph{minimum cyclic}.


Because $q_t = b_{i_{t-1}}$ (i.e., the state at $t$ depends on only the action played at $t-1$), playing a sequence $\{i_1, i_2, \dots, i_{N'}\}$ repeatedly $M$ times gives the same reward in expectation times $M$. In other words, if $\{\tq_1, \dots, \tq_{N'}\}$ is the sequence of induced states as a result of playing $\{i_1, \dots, i_{N'}\}$, and $q_t'$ the sequence of induced states for repeating $\{i_1, \dots, i_N'\}$ for $M$ times, we have that: 
\begin{equation}\label{eq:repeat}
\sum_{t \in [N' \cdot M]} q_t' \cdot r_{i_t} = \sum_{m \in [M]} \sum_{t \in [N']} q_t' \cdot r_{i_t} = \sum_{m \in [M]} \sum_{t \in [N']} \tq_t \cdot r_{i_t} = M \cdot R_1
\end{equation}
Assume now that $\pi^{\star}$ is \emph{not} comprised by the repetition of a minimum cyclic sequence. We use $\{\pist\}_{t_1}^{t_2}$ to denote the arms chosen by the optimal sequence between rounds $t_1$ and $t_2$, i.e., $\{\pist\}_{t_1}^{t_2} := \{ \pi^{\star}_{t_1}, \dots, \pi^{\star}_{t_2}\}$. From the assumption that $\pist$ is \emph{not} comprised by the repetition of a minimum cyclic sequence, we have that: $\{ \pi^{\star} \}_{t}^{t+N} \neq \{\pi^{\star}\}_{t+N+1}^{t+2N}$. Hence, playing the sequence: $$ S = \{ 
 \pi^{\star} \}_{1}^{t} + 2 \cdot \max \left(\{\pi^{\star}\}_{t}^{t+N},\{\pi^{\star}\}_{t+N+1}^{t+2N}\right) +\{ \pi^{\star} \}_{t+2N +1}^{T} $$ 
should be giving higher reward than $\pist$ gives, which is a contradiction. We have thus so far proved that $\pist$ is comprised by the repetition of a minimum cyclic sequence.

We next prove that $N \leq 2$. To do this, we use the \emph{rearrangement inequality}, which states that for \emph{any} two sequences $\{x_i\}_{i \in [n]}$ and $\{y_i\}_{i \in [n]}$ of real numbers such that:
\[
x_1 \leq \dots \leq x_n \quad \& \quad y_1 \leq \dots \leq y_n
\]
it holds that:
\begin{equation} \label{eq:rearr}
    x_n y_1 + \dots + x_1 y_n \leq x_{\sigma(1)}y_1 + \dots + x_{\sigma(n)}y_n \leq x_1 y_1 + \dots + x_n y_n
\end{equation}
where $\sigma(\cdot)$ is \emph{any} possible ordering. 

Assume that the optimal cycle consists of $N>2$ arms and that without loss of generality:
$$ r_N \geq r_{N-1} \geq ... \geq r_2 \geq r_1 $$
Thus, from \cref{eq:rearr} the optimal sequence $\pi^{\star}_N$ obtains reward:
\begin{equation}\label{eq:rearrange-1}
\OPT \leq \frac{T}{N} \left( r_N \cdot b_{\max} + r_{N-1} \cdot b_{\max-1} + \dots + r_1 \cdot b_{\min}\right)
\end{equation}
where $b_{\max} = \max_{i} b_i, b_{\min} = \min_{i} b_i$. From the rightmost side of the re-arrangement inequality, the equality in Equation~\eqref{eq:rearrange-1} is obtained when 
$$ b_1 \geq b_N \geq b_{N-1} \geq \dots \geq b_3 \geq b_2$$
Thus, the optimal cycle of $N>2$ arms will be:
$$ r_N b_1 + r_{N-1}b_N + r_{N-2}b_{N-1} + \dots + r_2b_3 +r_1 b_2.$$
Let ($i^{\star}, j^{\star}$) be the optimal cycle of size $2$:
$$ r_{\ist} b_{\jst} + r_{\jst} b_{\ist} \geq r_x b_y + r_y b_x, \ \forall x,y \in [K].$$
However, $r_{\ist} b_{\jst} + r_{\jst} b_{\ist} \geq r_Nb_1 + r_1b_N \geq r_N b_1 + r_1 b_2$ and $r_{\ist} b_{\jst} + r_{\jst} b_{\ist} \geq r_{i-1}b_{i} + r_{i}b_{i-1} \geq r_{i-1}b_{i} + r_{i-2}b_{i-1},$ $\forall i \in \{3,4, \dots, N \}$. Thus, $T/2 (r_{\ist} b_{\jst} + r_{\jst} b_{\ist}) \geq \OPT$, which is a contradiction. 
\end{proof}

Using the structure of the optimal sequence, we design our algorithm for the ``sticky'' arms case. We first discuss an example so as to give intuition regarding our algorithm.

Consider a setting where we have $K=3$ arms $\{A, B, C\}$ with reward tuples $(r_A, b_A) = (1, 0), (r_B, b_B) = (0, 1), (r_C, b_C) = \frac{1}{\sqrt{2}}(1+x, 1+x)$ where $x \in [-\Delta, \Delta]$ for some scalar $\Delta > 0$. As we proved (Lemma~\ref{lem:cycle-2}), the optimal policy consists of at most $2$ arms. It is easy to see that for our stated example, depending on whether $x > 1$, the optimal policy is either cycle $AB$ or just arm $C$. So now the main challenge in designing a policy which naively switches between $AB$ and $C$ such as:
\begin{equation*}
    \underbrace{ABA-CCC}_{\calC}-\underbrace{ABA-CCC}_{\calC}-\underbrace{ABA-CCC}_{\calC}-\cdots
\end{equation*}
is that the expected reward of cycle $\calC$ is: $(r_B b_A + r_A b_B) + (r_C b_A + r_A b_C) + 2r_C b_C$. If one were to define the \emph{meta-arms} ``$AB$'' and ``$CC$'', note that these satisfy $r_B b_A + r_A b_B = 1$ and $2 r_C b_c = (1+x)^2$. However, $r_A b_C + r_C b_A = (1+x)/\sqrt{2}$. This basically means if we are using an arm-elimination idea to distinguish the two meta-arms $AB$ and $CC$, each transition (i.e., switch between $A$ to $C$ or $C$ to $A$) is going to cost us a constant regret. This would lead to linear regret $\Omega(T)$. 

To drop the exponent to $1/2$, we still use meta-arms $AB$ and $CC$ but we come up with a strategy to minimize switches. To do so, we define \emph{batches} of meta-arms' being played. In our running example, the batches would be defined as follows: 
\begin{equation}\label{eq:redd}
   \underbrace{ABABA \cdots AB{\color{red}A}-{\color{red}C}CC \cdots  C}_{\text{Batch } 1}- \underbrace{ABABA \cdots AB{\color{red}A}-{\color{red}C}CC \cdots C}_{\text{Batch } 2}-\cdots
\end{equation}
This idea can be generalized using intuition from batched bandits~\citep{esfandiari2019regret} to more complex settings that contain more arms with arbitrary reward tuples $(r,b)$. We first create $K(K+1)/2$ meta-arms. Our meta-arms consist of pairs $\{(i \diamond j) \mid i \leq j \in [K] \}$. In our above example with $K = 3$, our $6$ meta-arms would be $\{ (A \diamond A), (B\diamond B), (C \diamond C), (A \diamond B), (A \diamond C), (B \diamond C) \}$. Note that from Lemma~\ref{lem:cycle-2} one of the above  meta-arms is the optimal one. A careful analysis of the generalized batched bandits algorithm would give $\DES$ regret: $\regDES(T) = \tcalO(K \sqrt{T})$.

Roughly, the reason for picking up the $K^2$ factor is that the immediate application of the batched bandits algorithm throws away some reward samples obtained. In our example with arms $\{A,B,C\}$, the samples that are underutilized are the ones in red in Eq.~\eqref{eq:redd}. Our final algorithm is able to shave off an extra $K$ factor, by not throwing away these samples. This is done by a more careful exploration algorithm described in Algorithm~\ref{algo:policy-meta-arms}. For the purposes of the analysis, we call the underutilized samples \emph{``switches''}.

\begin{algorithm}[t]
\SetAlgorithmName{Algorithm}{}{}
\caption{Smart Meta-Arm Switch Exploration}\label{algo:policy-meta-arms}
\DontPrintSemicolon
\LinesNumbered
\SetAlgoNoEnd
\textbf{Input.} Set of meta arms $\mathcal{A}$, rounds $U_{\beta}.$ \;
Initialize set of unexplored active arms: $\calA' \gets \calA$.\;
\While{$\mathcal{A'} \neq \emptyset$}{
    Choose a random meta-arm $(i \diamond j) \in \calA.$ \;
    Change arms by playing arm $i$. \tcp*{``Initialize'' meta-arm by playing $i$.}
    \tcc{Explore reward of meta-arm $(i \diamond j)$.}
    \For {$U_{\beta}$ rounds}{
        Play arm $j$ and observe reward.\;
        Play arm $i$ and observe reward. \;
    }
    Update unexplored active meta-arms: $\mathcal{A'} \leftarrow \mathcal{A'} \setminus {(i \diamond j)}$. \;
    \While {$\exists (x \diamond i)$ or $(i \diamond x)$ pair $ \in \mathcal{A}'$}{
    \tcc{Not throwing away the last observation of $i$, explore all meta-arms that include $i$.}
       \For {$U_{\beta}$ rounds}{
        Play arm $x$ and observe reward.\;
        Play arm $i$ and observe reward.\;
        }
        Update unexplored active meta-arms: $\mathcal{A}' \leftarrow \mathcal{A}' \setminus {(x \diamond i)}$. \;
    }
}
\end{algorithm}%

\begin{lemma} \label{lem:meta-arms-policy}
Algorithm~\ref{algo:policy-meta-arms} makes at most $K$ \emph{``switches''} of meta-arms.
\end{lemma}

\begin{proof}
Pick any $(i \diamond j) \in \calA'$. Algorithm~\ref{algo:policy-meta-arms} makes $1$ ``switch'' for the first pair $(i \diamond j)$. Assume that the last action played was action $i$. Then, to explore the meta-arms that are still in $\calA'$ and include $i$ the algorithm does \emph{not} make \emph{any} ``switch'' (i.e., does not throw away any samples). This is repeated until there are no more meta-arms in $\calA'$. There are at most $K(K+1)/2$ meta-arms and at most $K$ different arms in these meta-arms. Thus, Algorithm~\ref{algo:policy-meta-arms} does at most $K$ switches for each batch $B$.
\end{proof}

\begin{algorithm}[t]
\SetAlgorithmName{Algorithm}{}{}
\caption{Batched $\BDES$ for ``Sticky'' Arms}\label{algo:batched-bandit-meta-arms}
\label{algo:sticky}
\DontPrintSemicolon
\LinesNumbered
\SetAlgoNoEnd
\textbf{Input.} Number of batches $B = 2 \log T$, $K$ arms, time horizon $T$. \; 
Set $w=T^{1/B}=\sqrt{e}$ and generate $M = K(K+1)/2$ meta-arms $(a_i \diamond a_j)$, with $i \leq j, (i,j)\in[K]^2$. \; 
Set active meta-arms $\mathcal{A} = \{(i \diamond j)\mid i \leq j, i, j \in [K]\}$. \tcp*{$|\mathcal{A}|=M$ initially}
For $i \leq j \in [K]$ initialize estimated means $\hat{\mu}_{(i \diamond j)} = 0$. \;
\For{batch $\beta=1$ to $B-1$}{
\If{$\lfloor{w^\beta \rfloor} \cdot |\mathcal{A}| >$ remaining rounds}{
\bf{Break}
}
Play all $(i \diamond j) \in \mathcal{A}$ for $U_\beta = \lfloor{w^\beta \rfloor}$ times according to Algorithm~\ref{algo:policy-meta-arms}. 
\tcp*{contains $2 U_\beta + 1$ actions}
Drop the first reward observation that has expectation $r_{i} q_0$, where $q_0 = 1$. \;
Pair the other observations into $U_\beta$ groups of size $2$. \;
Update $\hat{\mu}_{(i \diamond j)}$ using these new $U_\beta$ observations (sample mean). \;
Update the number of observations of all existing meta-arms according to $c_\beta = \sum_{l=1}^\beta U_l$.

\For{each active arm $(i \diamond j)$ in $\mathcal{A}$}{Eliminate this arm if it is sub-optimal, i.e., remove it from $\mathcal{A}$ if it satisfies
\begin{center}
    $\hat{\mu}_{(i \diamond j)} < \max_{(u \diamond v) \in \mathcal{A}} \hat{\mu}_{(u \diamond v)} - \sqrt{\frac{2 \log(2 K^2 T B)}{c_\beta}}$
\end{center}
}
}
In the last batch, play the optimal remaining meta-arm, i.e., the one that has the highest $\hat{\mu}_{(i \diamond j)}$.
\end{algorithm}%

We are now ready to sketch the proof of the upper bound for Theorem~\ref{thm:regret-sticky}.

\begin{proof}[Proof Sketch of Theorem~\ref{thm:regret-sticky} for $\lambda = 1$.]
The average of rewards observed by meta-arm $(i \diamond j)$ satisfies $\E [ (\tilde r_t + \tilde r_{t+1})/2] = (r_i b_j + r_j b_i)/2$, since we have i.i.d. and $\sigma/\sqrt{2}$-subgaussian observations.

Let $(\ist \diamond \jst)$ be the optimal meta-arm. Let $\Delta_{(i \diamond j)}$ be the gap of meta-arm $(i\diamond j)$, i.e., $\Delta_{(i \diamond j)} = {r_{\ist} b_{\jst} + r_{\jst} b_{\ist} - r_{i} b_{j} - r_{j} b_{i}}/{2}$. Then, the regret incurred throughout $T$ rounds can be written as:
\begin{align*}
    \regDES(T) &= \sum_{t=1}^T \left(\frac{r_{\ist} b_{\jst} + r_{\jst} b_{\ist}}{2} - r_{I_t} b_{I_{t-1}}\right) \\
    &\leq \sum_{1 \leq i \leq j \leq K} \Delta_{(i \diamond j)} N_{(i \diamond j)} + \sum_{t \in [T]} \mathbb{I}[\text{transition between two meta-arms happens at $t$}]
\end{align*}
where $N_{(i \diamond j)}$ is the number of pulls of meta-arm $(i \diamond j)$ during $T$ rounds. Since $r,b \in [0,1]$, then the second term in the above is upper bounded by $B K$ (Lemma~\ref{lem:meta-arms-policy}). So the regret is upper bounded by:
\begin{equation}\label{eq:regr-sticky-app}
    \regDES(T) \leq 2 \sum_{1 \leq i \leq j \leq K} \Delta_{(i \diamond j)} N_{(i \diamond j)} + BK \,.
\end{equation}
The remainder of the proof is to bound $N_{(i \diamond j)}$, which is based on the arm-elimination protocol. Specifically, we show that for a meta-arm $(i \diamond j)$ that was not eliminated at batch $\beta$, we have that $ \Delta_{(i \diamond j)} \leq 2 \sqrt{{2 \log(2K^2 B T)}/{ c_\beta}}$, which means that
\begin{equation*}
    N_{(i \diamond j)} \leq c_{\beta+1} =  w + w c_\beta=  w + 8 w \log\left(2K^2 BT\right) \Delta_{(i \diamond j)}^{-2}  \,.
\end{equation*}
After parameter tuning, we get the result. We include the full proof in Appendix~\ref{app:sticky-arms}.
\end{proof}

\subsection{Evolution Rate $ \lambda \in [\widetilde{\Theta}(1 - 1/\sqrt{T}), 1)$}

When the evolution rate is equal to $\lambda = 1 -\epsilon$, then after playing an arm $I_t$ the state becomes \emph{almost} $I_t's$ baseline reward: $q_{t+1} = \epsilon q_{t} + (1 - \epsilon)\cdot b_{I_t}$. Thus, playing the best meta-arm as defined in the previous subsection in near optimal.

\begin{lemma}
For $\lambda \in [\widetilde{\Theta}(1 - 1/\sqrt{T}), 1)$, Algorithm~\ref{algo:batched-bandit-meta-arms} incurs regret $R(T) = \calO (K\sqrt{T \log (KT)})$.
\end{lemma}

\begin{proof}
Let $\{ I_t \}_{t \in [T]}$ be the sequence of arms played by Algorithm~\ref{algo:batched-bandit-meta-arms} and $( 
\ist \diamond  \jst)$ be the best meta-arm. Note that Algorithm~\ref{algo:batched-bandit-meta-arms} treats the setting ``pretending'' that $\lambda = 1$, so it mis-estimates the best-fixed meta-arm by a factor of \emph{at most} $(1 - \lambda) \leq \widetilde{\Theta}(1 / \sqrt{T})$ at each round. This is essentially because the state is misestimated at each round by an $\epsilon \leq \widetilde{\Theta}(1/\sqrt{T})$. Hence, the regret incurred is: 

$$\regDES(T) \leq \calO\left(K \sqrt{T \log (KT)} \right) + (1/\sqrt{T} )\cdot T \leq \calO \left( K \sqrt{T \log (KT)} \right).$$
\end{proof}

\section{Robustness}\label{sec:robustness}

We show next that the results of the previous sections are \emph{robust} to two types of model misspecifications; first, that the state augmented reward is not deterministically decided by $q_t$, but instead, there is some stochastic noise that affects it; second, that the state evolution parameter $\lambda$ is originally unknown to the principal. The proofs and supplementary material can be found in \cref{app:robustness}.

\subsection{Noise Perturbed Model}

We focus on the following model for the noisy states transition: while the actual transition remains deterministic (i.e., $q_{t+1} = (1 - \lambda)q_t + \lambda b_{I_t}$), the reward at round $t+1$ is sampled from $\texttt{Bern}(r_{I_t} \tilde{q}_t(\nu_t))$, where $\tilde{q}_{t}(\nu_t) = q_t + \nu_t$, and $\nu_t$ is a noise random variable drawn from a $\sigma$-subGaussian distribution $\mathcal{D}$. This model (which we refer to as the ``noise-perturbed'' model) captures misspecifications in how the current state affects the per-round reward. The principal does \emph{not} need to know the noise distribution or the variance. Essentially, we prove that our algorithms are robust to such noisy states. Note that in the noise-perturbed model, the regret definition changes as follows: 
\begin{center}
    $R_{\DES}(T) = \E_{\nu_t \sim \mathcal{D}} \left[ \max_{\tilde{\pi}^{\star}} \sum_{t \in [T]} r_{\tilde{\pi}_t^{\star}} \tilde{q}_t(\nu_t) - \sum_{t \in [T]} r_{I_t} q_t \right].$
\end{center}
where the sequence $\tilde{\pi}^{\star}$ is the optimal sequence of arms assuming that the benchmark had access the noise distribution $\mathcal{D}$, but not to the actual realizations $\nu_t$. Formally, we prove the following. 

\begin{theorem}\label{lem:noise-perturbed}
For the noise-perturbed model: (i) if $\lambda \in [0, \widetilde{\Theta}(1/T)]$ then Algorithm \textsf{EXP3.P} incurs regret $R_\DES(T) = \calO(\sqrt{KT \log K} + \sigma T + (1 - (1 - \lambda)^T)\cdot \OPT)$; (ii) if $\lambda \in (\widetilde{\Theta}(1/T), \widetilde{\Theta}(1 - 1/\sqrt{T}))$, then Algorithm~\ref{algo:ETC-known-iR} incurs regret $R_{\DES}(T) = \tcalO \left( \left(\frac{K \log (T) \log (\lambda)}{\log (1 - \lambda)}\right)^{1/3} \cdot T^{2/3} + \sigma T \right)$; (iii) if $\lambda \in [\widetilde{\Theta}(1 - 1/\sqrt{T}), 1]$, then \cref{algo:sticky} incurs regret $R_{\DES}(T) = \tcalO( K \sqrt{T} + \sigma T)$. All algorithms are agnostic to $\sigma$.
\end{theorem}

\subsection{Unknown $\lambda$}

When $\lambda$ is originally unknown, we find an algorithm (\cref{algo:unknown-lambda} in the Appendix) that can guarantee sublinear regret under one of the two assumptions: (A1) $\max_{i,j \in [K]} r_i |b_i - b_j | > \tilde{\omega}(1/T^{-1/3})$ {\bf or} (A2) $\lambda \notin (\widetilde{\Theta}(1/T), \widetilde{\Theta}(K^{1/3}/T^{1/3})]$. Assumption (A1) is a ``discrepancy assumption'' that intuitively says that there exist two arms whose ES are more than $1/T^{1/3}$ away. 

\begin{theorem}[Informal]
    Under either assumption (A1) or (A2), \cref{algo:unknown-lambda} incurs regret $\widetilde{\calO}({K^{1/3} T^{2/3}})$.  
\end{theorem}

At the heart of \cref{algo:unknown-lambda} lies the following idea: if $\lambda$ is close to $1$, then alternating between two randomly chosen arms for two ``epochs'', the realized rewards from the two epochs will be sufficiently (i.e., up to factors that depend on how close to $1$ $\lambda$ is and Hoeffding bounds) close. On the other hand, if $\lambda$ is close to $0$, then if you first drive the state to be approximately $b_i$ (by playing repeatedly arm $i$, see \cref{lem:state-approx-bi}) and then take enough alternating samples from arms $i,j$, your realized rewards should be close to the realized reward only for one of the two arms. If $\lambda$ is not near $0$, then do a binary search on the $N(\lambda)$ until the algorithm detects that the state has converged. Once we have our estimator for $\lambda$, we call Algorithm~\ref{algo:ETC-known-iR}. Although the full proof is very involved, to get the regret bound, one puts these ideas together, and tunes the epoch lengths to collect enough samples. Note that either (A1) or (A2) are needed in our analysis in order to guarantee that with a sublinear number of samples we have been able to distinguish between the effects of an unknown $\lambda$ versus the effect of sampling from unknown reward distributions.

\section{Discussion}\label{sec:limitations-blindness}

In this paper, we studied a bandit learning setting which accounts for long-term effects and whose main applications are learning for online advertising and recommendation systems. Central to our construction is the notion of the \emph{state} and the \emph{state evolution rate} $\lambda$, which captures how fast the system evolves. 

There are three avenues for future research on this space. First, and on a more technical note, the most important open question is providing algorithms with sublinear regret guarantees for the case where $\lambda = \Theta(1/T)$. Recall that for the case where $\lambda = \Theta(1/T)$, we can only prove $(1 - 1/e)$-\emph{approximate} $\DES$ regret; in other words, the cumulative reward it obtains is close to $(1-1/e)$ times the cumulative reward of the benchmark policy. The intuition behind the hardness that arises in this case is that $\lambda = \Theta(1/T)$ is an in-between regime where the changes in states happen ``fast'' enough for the approach of Section~\ref{sec:small-lambda} to not work but ``slow'' enough for the algorithm of Section~\ref{sec:gen-evol-rate} to not have enough samples to construct the estimates $\hr_i, \hb_i$. We think that there is hope to address the challenge by using a completely different approach, at least for specific instances of problems.

For the second avenue for future research, it is an open question how to obtain better bounds for the agnostic $\lambda$ case. Specifically, we think that a very interesting idea would be to try to obtain optimal regret bounds when you are given an original \emph{prediction} about how big $\lambda$ is, while being robust to potentially adversarial information. The hardness that one would need to overcome would be to find a way to distinguish between good and bad predictions, without having to resort to our approach that alternates a specific set of arms until we have converged.

The third avenue is related to richer models in this space of evolving preferences. For example, studying a ``contextual'' or multi-dimensional version of bandit learning with long-term effects is a particularly intriguing question. It is currently unclear (even from a modeling perspective) how the interplay between contexts and states would change the regret rates obtainable in this case.

\bibliographystyle{icml2023}
\bibliography{refs}

\newpage
\appendix
\section{Appendix for Section~\ref{sec:model-blindness}}\label{app:model}

\begin{proof}[Proof of \cref{lem:external-vs-DES}]
    We are going to prove the lemma with an instance of sticky arms, i.e., $\lambda = 1$. Specifically, let us define instance $\calI$ as a sticky arms problem with 2 arms, for which it holds that: $(r_1, b_1) = (1/2, 1)$ and $(r_2, b_2) = (3/4 - \eps/2, 1/2 + 2\eps)$, for some parameter $\eps > 0$ to be specified later. Observe that for $\calI$ the optimal sequence of arms to be played contains both arm $1$ and $2$. This is because: 
    \begin{itemize}
    \item If arm $1$ was the only one to be played repeatedly, then the expected reward collected per round would be $R_1 = r_1 b_1 = 1/2$.
    \item If arm $2$ was the only one to be played repeatedly, then the expected reward collected per round would be $R_2 = r_2 b_2 = 3/4 + 3\eps/2 - 2\eps^2$.
    \item Finally, if arms $1, 2$ were to be played repeatedly one after the other, then the expected reward collected per round would be: $R_{12} = 1 + \eps/2$.
    \end{itemize}
    As a result, for the $\regDES(T)$ the benchmark sequence is playing arms $1, 2$ alternatively for $T$ (so $T/2$ rounds per arm) rounds and collects expected reward: $(1 + \eps/2)T$.

    Let us now think of an algorithm $\ALG$ that minimizes \emph{external} regret on instance $\calI$. Since $R_2 > R_1$, the best-fixed arm in hindsight for the external regret is arm $2$. This in turn means that any algorithm that has sublinear external regret must play arm $2$ at least $T - o(T)$ times and arm $1$ at most $o(T)$ times. Let $\Sigma = (a_1, a_2, \dots, a_T)$ the sequence of arms chosen by $\ALG$.
    \begin{itemize}
        \item If arm $1$ is played after arm $1$, then the expected reward is: $R_{11} = 1/2$.
        \item If arm $1$ is played after arm $2$, then the expected reward is $R_{12} = 1 + \eps/2$
    \end{itemize}
    On sequence $\Sigma$, arm $1$ is played at most $o(T)$ times. Assume that $b \in [0,1]$ fraction of these, arm $1$ was played after an arm $1$ pull and respectively $(1-b)$-fraction of times, arm $1$ was played after an arm $2$ pull. For arm $2$ (which is played $T - o(T)$ times) assume $c \in [0,1]$ fraction of these it 
    As a result, the expected reward of $\ALG$ would be:
    \begin{align*}
    Rew(\ALG) &= b \cdot o(T) \cdot R_{11} + (1 - b) \cdot o(T) \cdot R_{12} + c \cdot (T - o(T)) \cdot R_{22} + (1 - c)(T - o(T)) R_{12}
    \end{align*}
    In the best case scenario (i.e., the one that gives the highest possible reward to $\ALG$) for $\Sigma$ every pull of arm $1$ was preceded by a pull of arm $2$; this means that $o(T)$ pulls from arm $2$ have been used and no more pulls from arm $1$ are left in the sequence $\Sigma$. This means that $c \approx 1$ and the reward for $\ALG$ is: 
    \[
    Rew(\ALG) \leq o(T) \left( 1 + \frac{\eps}{2} \right) + (T - o(T)) \left( \frac{3}{4} + \frac{3\eps}{2} - 2 \eps^2\right) \leq 3 o(T) + \frac{3T}{4} + \frac{3\eps T}{2}
    \]
    As a result, on sequence $\Sigma$ algorithm $\ALG$ incurs regret: 
    \[
    \regDES(T) = \left( 1 + \frac{\eps}{2} \right)T - Rew(\ALG) \geq \frac{T}{4} - \frac{3\eps T}{2} - 3 o(T) = \Omega(T)
    \]
\end{proof}

\section{Appendix for Section~\ref{sec:gen-evol-rate}}

\subsection{Missing Proofs for Section~\ref{sec:DP-relax}}\label{app:DP-relax}

\begin{proof}[Proof of Lemma~\ref{lem:state-closed-form}]
We prove the lemma using induction. For the base case $t = 1$, from Equation~\eqref{eq:qt-def} it holds that 
\[q_1 \left(H_{1:1}^\ALG \right) = (1 - \lambda) \cdot q_0 \left( H_0^\ALG \right) + \lambda \cdot b_{I_1} = (1 - \lambda) + \lambda \cdot b_{I_1},
\]
which is equal to $q_{t+1}(H_{1:t}^\ALG) = (1-\lambda)^1 \cdot q_0 + \lambda \cdot (1 - \lambda)^0 \cdot b_{I_1}$ from Equation~\eqref{eq:qt-closed-form}.

For the inductive step, assume that Equation~\eqref{eq:qt-closed-form} holds for some $t = n$. Then, for $t = n+1$ from Equation~\eqref{eq:qt-def} we have that: 
\begin{align*}
    q_{n+2} &\left(H_{1:n+1}^\ALG \right) = (1 - \lambda) \cdot q_{n+1} \left( H_{1:n}^\ALG \right) + \lambda \cdot b_{i_{n+1}} \\
            &= (1 - \lambda) \left[ (1 - \lambda)^{n+1}\cdot q_0 + \lambda \cdot \sum_{s=0}^n (1 - \lambda)^{n-s} \cdot b_{I_s}\right] + \lambda \cdot b_{I_{n+1}} &\tag{inductive step} \\ 
            &= (1 - \lambda)^{n+2} \cdot q_0 + \lambda \cdot \sum_{s=0}^n (1 - \lambda)^{n+1 -s} \cdot b_{I_s} + \lambda \cdot b_{I_{n+1}} \\
            &= (1-\lambda)^{n+2} \cdot q_0 + \lambda \cdot \sum_{s=0}^{n+1} ( 1- \lambda)^{n+1 - s} \cdot b_{I_s}
\end{align*}
which is exactly the form that $q_{n+2}(H_{1:n+1}^\ALG)$ takes from Equation~\eqref{eq:qt-closed-form}. This concludes our proof.
\end{proof}

\begin{proof}[Proof of Lemma~\ref{lem:approx-DP}]
The solution to the DP algorithm achieves the following reward: 
\begin{align*}
    \hDP &= \max_{i_1, \dots, i_T} \sum_{t \in [T]} \left[ (1 - \lambda)^t + \lambda \sum_{s = 0}^{t - 1} (1 - \lambda)^{t-1-s} \hb_{i_s}\right] \cdot \hr_{i_t} &\tag{Equation~\eqref{eq:qt-closed-form}}\\
    &\geq \max_{i_1, \dots, i_T} \sum_{t \in [T]} \left[ (1 - \lambda)^t + \lambda \sum_{s = 0}^{t - 1} (1 - \lambda)^{t-1-s} \left(b_{i_s} - \delta \right)\right] \cdot \left(r_{i_t} - \delta \right) &\tag{$\hr_i \geq r_i - \delta, \hb_i \geq b_i - \delta$}\\
    &\geq \max_{i_1, \dots, i_T} \sum_{t \in [T]} \left[ (1 - \lambda)^t + \lambda \sum_{s = 0}^{t - 1} (1 - \lambda)^{t-1-s} b_{i_s} \right] \cdot r_{i_t} - \sum_{t \in [T]} \lambda \sum_{s =0}^{t-1}(1-\lambda)^{t - 1 - s} \cdot \delta \\
    &\geq \max_{i_1, \dots, i_T} \sum_{t \in [T]} \left[ (1 - \lambda)^t + \lambda \sum_{s = 0}^{t - 1} (1 - \lambda)^{t-1-s} b_{i_s} \right] \cdot r_{i_t} - \sum_{t \in [T]} \lambda \cdot \frac{1}{\lambda} \cdot \delta \\ 
    &\geq \sum_{t \in [T]} \left[ (1 - \lambda)^t + \lambda \sum_{s = 0}^{t - 1} (1 - \lambda)^{t-1-s} b_{\pi^{\star}_s} \right] \cdot r_{\pi^{\star}_t} - \delta \cdot T &\tag{properties of $\pist$}\\
    &= \OPT - \delta T
\end{align*}
where the second inequality also uses the fact that $r_i, b_i \in [0,1]$.
\end{proof}

 \begin{proof}[Proof of Lemma~\ref{lem:FPTAS}]
    Let $S$ be the sequence of arms the \cref{algo:fptasdp} returns and $\pi^{\star}$ the optimal sequence of arms as usual. First, we will prove that achieves $(1 - \epsilon)$ approximation and then that it has time complexity $\calO(K T^2 / \epsilon)$. It holds that $\epsilon  \left \lfloor \frac{1}{\epsilon} r_i \cdot q \right \rfloor \leq r_i \cdot  q \leq \epsilon ( \left \lfloor \frac{1}{\epsilon} r_i \cdot q \right \rfloor + 1) $ and $|S| = |\pi^{\star}| = T$. Hence: 
\begin{align*}
    \sum_{i \in S} r_i q & \geq \epsilon \sum_{i \in S}   \left \lfloor \frac{1}{\epsilon} r_i \cdot q \right \rfloor \geq \epsilon \sum_{i \in \pi^{\star}}   \left \lfloor \frac{1}{\epsilon} r_i \cdot q \right \rfloor \geq \sum_{i \in \pi^{\star}} r_i q - \epsilon \OPT \geq \OPT - \epsilon \OPT = (1 - \epsilon) \OPT.
\end{align*}
Regarding the runtime, note that the sequence of tuples $F_{t}$ at the end of each round $t$ has at most $T / \epsilon$ tuples. Thus, the  time complexity of \cref{algo:fptasdp} is:
$$ \calO \left( T \cdot K \cdot  \frac{1}{\epsilon} T \right)  = \calO\left (\frac{1}{\epsilon} K T^2 \right).$$
\end{proof}

\subsection{Missing Proofs for Section~\ref{subsec:gen-lam}}\label{app:gen-lambda}

\begin{proof}[Proof of Lemma~\ref{lem:state-approx-bi}]
Let $\REP_i$ be the algorithm that continuously plays arm $i$, and let $\widetilde{H}_{s:t}^{\REP_i} = H_{s:t}^{\REP_i} \cup H'$. We first prove by induction that if $i = i_\tau, \forall \tau \in \{1, \dots, N(\lambda) \}$, then: 
\begin{equation}\label{eq:NR1}
q_{s+\tau+1}\left(\widetilde{H}_{s:s+\tau}^{\REP_i} \right) - b_{i} = (1 - \lambda)^{\tau+1} \left( q_s - b_{i} \right).
\end{equation}
For the base case $\tau=1$, note that $q_{s+1}(H') - b_{i} = (1-\lambda)(q_s - b_{i})$, which is equal to the definition in Eq.~\eqref{eq:qt-def}, if the first round was $s$ instead of $1$. For the inductive step, assume for $\tau = n$:
\begin{equation}\label{eq:step}
q_{s+n+1}\left(\tH_{s:s+n}^{\REP_i} \right) - b_{i} = (1 - \lambda)^{n+1} \left(q_s - b_{i} \right)
\end{equation}
Then, for $\tau = n+1$, from Eq.~\eqref{eq:qt-def}, we have:
\begin{align*}
q_{s+n+2}\left(\tH_{s:s+n+1}^{\REP_i} \right) &= (1 - \lambda) q_{s+n+1}\left(\tH_{s:s+n}^{\REP_i} \right) + \lambda b_{i} \Leftrightarrow \\ q_{s+n+2}\left(\tH_{s:s+n+1}^{\REP_i} \right) - b_{i} &= (1 - \lambda) \left( q_{s+n+1}\left(\tH_{s:s+n}^{\REP_i} \right) - b_{i} \right)
\end{align*}
Substituting Equation~\eqref{eq:step} in the latter completes the proof of the induction.

To simplify notation, we use $q_{\tau+1} = q_{s+\tau+1}(\tH_{s:s+\tau}^{\REP_i})$. Taking the absolute on both sides of Eq.~\eqref{eq:NR1}: 
\begin{equation*}
    \left|q_{N(\lambda)} - b_i \right| = \left| (1-\lambda)^{N(\lambda)}(q_s - b_i) \right|
\end{equation*}
Substituting the expression for $N(\lambda)$ from the lemma statement, we get: 
\begin{align*}
     \left|q_{N(\lambda)} - b_i \right| &= \left| (1-\lambda)^{\frac{\log (\lambda \eps)}{\log (1 - \lambda)}}(q_s - b_i) \right| \leq \left| (1-\lambda)^{\frac{\log (\lambda \eps)}{\log (1 - \lambda)}}\right| \cdot \left|q_s - b_i \right| &\tag{Cauchy-Schwarz} \\
     &\leq \left| (1-\lambda)^{\frac{\log (\lambda \eps)}{\log (1 - \lambda)}}\right| &\tag{$q_s, b_i \in [0,1]$} \\
     &= 2^{{\frac{\log (\lambda \eps)}{\log (1 - \lambda)}}\cdot \log (1-\lambda)} = \lambda \eps \leq \eps &\tag{$\lambda \in (0,1)$}
\end{align*}
This concludes our proof.
\end{proof}

\begin{proof}[Proof of Corollary~\ref{lem:replenish-time}]
Similarly to the proof of Lemma~\ref{lem:state-approx-bi}, let $\REP_i$ be the algorithm that continuously plays arm $i$, and let $\tH_{s:t}^{\REP_i} = H_{s:t}^{\REP_i} \cup H'$. Then, from Equation~\eqref{eq:NR1} simplifying notation: $q_{N_R} = q_{s + N_R}(\tH_{s:N_R - 1}^{\REP_{i_R}})$, we get:
\begin{align*}
q_{N_R} - b_{i_R} = (1 - \lambda)^{N_R} \cdot (q_s - b_{i_R})
\end{align*}
Using the fact that $1 - \eps \leq b_{i_R} \leq 1$, the latter becomes: 
\[
q_{N_R} - (1 - \eps) \geq (1 - \lambda)^{N_R} \cdot (q_s - 1)
\]
Substituting for $q_0 = 1$ and $N_R$ as given in the lemma statement: $q_{N_R} - (1 - \eps) \geq \lambda \eps$. Re-arranging, we obtain the result.
\end{proof}

\begin{proof}[Proof of Lemma~\ref{lem:rew-estimator1}]
From Hoeffding's inequality on $\hr_i$ and using the fact that the block size is $M$ rounds, we get: 
\begin{equation}\label{eq:Hoeff-hri}
    \Pr \left[ \left| \hr_i - \E \left[ \hr_i \right] \right| \geq \delta \right] \leq 2 \exp \left( -2 M \delta^2 \right)
\end{equation}
From Corollary~\ref{lem:replenish-time}, regardless of the starting state and the prior history, if arm $i_R$ is played repeatedly for $N_R$ rounds, then at round $t_{j}^i$ the system's state is at $q_{t_{j}^i} \geq 1 - \eps$. So (by definition of our setting) the expected reward at the right next round (i.e., Line 9 of Algorithm~\ref{algo:ETC-known-iR}) is 
\[
\E [ R_j^i ] = q_{t_j^i} \cdot r_i \in [(1 - \eps) \cdot r_i, r_i ],
\]
with probability $1$. As a result, by the linearity of expectation and using the definition of $\hr_i$: 
\begin{align*}
\E \left[ \hr_i \right] &= \frac{\E \left[ R_j^i \right]}{M} = 
\frac{\sum_{j \in [M]} q_{t_j^i} \cdot r_i}{M} = r_i \cdot \frac{\sum_{j \in [M]} q_{t_j^i}}{M} \Rightarrow \E \left[ \hr_i \right] \in \left[ r_i\cdot (1- \eps), r_i \right]
\end{align*}
From Equation~\eqref{eq:Hoeff-hri}, we have that: 
\begin{align*}
    2 \exp \left( - 2 M \delta^2 \right) &\geq \Pr \left[ \hr_i - \E [\hr_i ] \geq \delta \; \text{or} \; \hr_i - \E[\hr_i] \leq - \delta \right] \\
    &\geq \Pr \left[ \hr_i \geq r_i + \eps + \delta \; \text{or} \; \hr_i - \E[\hr_i] \leq - \delta \right] &\tag{$\E[\hr_i] \leq r_i \leq r_i + \eps$}\\
    &\geq \Pr \left[ \hr_i \geq r_i + \eps + \delta \; \text{or} \; \hr_i \leq r_i - \eps - \delta \right] &\tag{$\E[\hr_i] \geq r_i \cdot (1 - \eps) \geq r_i - \eps$} \\ 
    &= \Pr \left[ \left| \hr_i - r_i \right| \geq \delta + \eps \right]
\end{align*}
Using as $\delta' = \delta+\eps$ and substituting in the above gives us the result.
\end{proof}

\begin{proof}[Proof of Lemma~\ref{lem:hv-estimator}]
The proof is similar to the proof of Lemma~\ref{lem:rew-estimator1}, but we include it here for completeness. From Hoeffding's inequality on $\hv_i$ and using the fact that the block size is $M$ rounds, we get:
\begin{equation}\label{eq:hoeffding-hvi}
    \Pr \left[ \left| \hv_i - \E \left[ \hv_i \right]\right| \geq \delta \right] \leq 2 \exp \left( -2 M \delta^2 \right)
\end{equation}
From Lemma~\ref{lem:state-approx-bi}, regardless of the history of plays, if you start from state $q_0$ and play the same arm for $N(\lambda)$ rounds, then the state becomes approximately equal to the baseline reward of that arm. In other words: $|q_{\tt_j^i} - b_i| \leq \eps$ and this means that: 
\[
\E \left[ S_j^i \right] = q_{\tt_j^i} \cdot r_i \in \left[ (b_i - \eps) \cdot r_i, (b_i + \eps) \cdot r_i \right] \Rightarrow \E \left[ S_j^i \right] \in [v_i, (1+\eps) \cdot v_i]
 \]
with probability 1. Note that the last derivation is because $v_i - \eps r_i \leq v_i$ and $v_i + \eps r_i \geq v_i + \eps v_i r_i$. As a result, by the linearity of expectation and using the definition of $\hv_i$:
\[
\E \left[ \hv_i \right] = \frac{\E \left[ S_j^i \right]}{M} = \frac{\sum_{j \in [M]} q_{\tt_j^i} \cdot r_i}{M} = r_i \cdot \frac{\sum_{j \in [M]} q_{\tt_j^i}}{M} \Rightarrow \E \left[ \hv_i \right] \in \left[ v_i, (1 + \eps) \cdot v_i \right]
\]
From Equation~\eqref{eq:hoeffding-hvi}, we have that:
\begin{align*}
    2 \exp \left( - 2 M \delta^2 \right) &\geq \Pr \left[ \hv_i - \E [\hv_i ] \geq \delta \; \text{or} \; \hv_i - \E[\hv_i] \leq - \delta \right] \\
    &\geq \Pr \left[ \hv_i \geq v_i + \eps + \delta \; \text{or} \; \hv_i - \E[\hv_i] \leq - \delta \right] &\tag{$\E[\hv_i] \leq v_i \leq v_i + \eps$}\\
    &\geq \Pr \left[ \hv_i \geq v_i + \eps + \delta \; \text{or} \; \hv_i \leq r_i - \eps - \delta \right] &\tag{$\E[\hv_i] \geq r_i \cdot (1 - \eps) \geq v_i - \eps$} \\ 
    &= \Pr \left[ \left| \hv_i - v_i \right| \geq \delta + \eps \right]
\end{align*}
\end{proof}

\begin{proof}[Proof of Theorem~\ref{thm:regret-ETC}]
For the rounds that pass while we are on lines 3 -- 17 of Algorithm~\ref{algo:ETC-known-iR}, we pick up regret at most $1$ at each of them. Hence, the regret picked up in total equals to the number of rounds between these lines which are $2 c (\lambda) \cdot \log (1 / \lambda \eps) \cdot K \cdot M$. 

We next define events $\calE_{r,i} = \{|\hr_i - r_i| \leq \delta \}$ and $\calE_{b,i} = \{ |\hb_i - b_i| \leq \delta \}$ for all $i \in [K]$. Then, conditional on the event $\calE = \{ \cap_{i \in [K]} (\calE_{r,i} \, \, \text{and} \, \, \calE_{v,i})\}$ and due to Lemma~\ref{lem:approx-DP}, the regret picked up for all the remaining rounds after feeding estimates $\{(\hr_i, \hb_i) \}_{i \in [K]}$ to the dynamic programming procedure is at most $\delta T$. As a result, from the law of total probability, the regret for all $T$ rounds: 
\begin{align*}
    \regDES(T) &\leq 2 K M \frac{\log (\lambda \eps)}{\log ( 1 - \lambda)} + \delta \cdot T \cdot \Pr \left[ \calE \right] + T \cdot \Pr \left[ \calE' \right] \\
    &\leq 2 K M \frac{\log (\lambda \eps)}{\log ( 1 - \lambda)} + \delta \cdot T + T \cdot \Pr \left[ \calE' \right] \numberthis{\label{eq:regr-full}}
\end{align*}
We next compute $\Pr [ \calE' ]$.
\begin{align*}
    \Pr &\left[ \calE' \right] = \Pr \left[ \cup_{i \in [K]} \left( \calE_{r,i}' \, \, \text{or} \, \, \calE_{v,i}' \right)\right]\\
    &\leq \sum_{i \in [K]} \left(\Pr \left[ \calE_{r,i}' \right] + \Pr \left[ \calE_{v,i}' \right] \right) &\tag{union bound}\\
    &\leq 6 K \exp \left( - 2M (\delta - \eps)^2 \right) +  4\exp \left( - 2M \cdot \left(\eps^2 - \eps \delta \right)\right) &\tag{Lemmas~\ref{lem:rew-estimator1},~\ref{lem:hb-estimator}}
\end{align*}
where the first derivation is because $\Pr[(A \cap B)'] = \Pr[A' \cup B']$.
Tuning $\delta = \eps / 4$ the latter becomes: $\Pr \left[ \calE' \right] \leq 8 K \exp \left( - M \eps^ 2 \right)$. Tuning $M = \log(T) / \eps^2$: $\Pr \left[ \calE' \right] \leq 8K/T$. As a result, the regret from Equation~\eqref{eq:regr-full} becomes: 
\[
\regDES (T) \leq 2 K \cdot \frac{\log (T) }{\eps^2}\cdot \frac{\log (\lambda \eps)}{\log ( 1 - \lambda)} + \frac{\eps}{4} \cdot T + 8K
\]
Tuning $\eps$ as stated gives us the result.
\end{proof}

\subsection{Generalization for Unknown Replenishing Arm} \label{app:gen-replenish-arm}

In this section, we show how the algorithm and the analysis for general $\lambda$ changes once the replenishing arm is not known or has baseline reward that is not within $[1-\eps, 1].$

\begin{lemma}\label{lem:balanced}
Any instance of $K$ $\BDES$ with $(r_i, b_i)_{i \in [K]}$ and an initial state $q_0$ is equivalent to an instance with tuples $(r_i', b_i') = (c r_i, b_i/c), \forall i \in [K]$ and initial state $q_0' = q_0/c$ for a constant $c > 0$.
\end{lemma}

\begin{proof}
To see this, note that the expected reward picked up after $T$ rounds by a sequence of actions $\{I_t\}_{t \in [T]}$ when the sequence of induced states is $\{q_t\}_{t \in [T]}$ is equal to: 
\begin{align*}
    \sum_{t \in [T]} q_t r_{I_t} &= \sum_{t \in [T]} (1 - \lambda)^t q_0 r_{I_t} + \lambda \sum_{t \in [T]} \sum_{s \in [t-1]} (1 - \lambda)^{t - 1 - s} b_{I_s} r_{I_t} &\tag{Lemma~\ref{lem:state-closed-form}} \\
    &=\sum_{t \in [T]} (1 - \lambda)^t \frac{q_0}{c} r_{I_t} c + \lambda \sum_{t \in [T]} \sum_{s \in [t-1]} (1 - \lambda)^{t - 1 - s} \frac{b_{I_s}}{c} r_{I_t}c \\
    &= \sum_{t \in [T]} (1 - \lambda)^t q_0' r_{I_t}' + \lambda \sum_{t \in [T]} \sum_{s \in [t-1]} (1 - \lambda)^{t - 1 - s} b_{I_s}'r_{I_t}'
\end{align*}
This concludes our proof.
\end{proof}

Next, we show how to choose $c$ in order to guarantee that there exists an arm whose baseline reward is inside $[1-\eps, 1]$. This is the ``replenishing'' arm in the general case.

\begin{lemma}\label{lem:replenish-factor}
Let $\ist = \arg\max_{i \in [K]} b_i$. Then, for any $\eps > 0$ choosing $c = b_{\ist} + \eps b_{\ist}$ guarantees that $b_{\ist}' \in [1-\eps, 1]$.
\end{lemma}

\begin{proof}
For the lower bound:
\begin{align*}
    b_{\ist}' = \frac{b_{\ist} }{b_{\ist}(1+\eps)} = \frac{1 }{1+\eps} > 1 - \eps \Leftrightarrow 1  > 1 + \eps -  \eps^2 \Leftrightarrow 0 > - \eps^2
\end{align*}
which is true. For the upper bound: 
\[
    b_{\ist}' = \frac{b_{\ist}}{b_{\ist}(1+\eps)} = \frac{1 }{1+\eps} < 1 
\]
\end{proof}

Moving forward, we assume without loss of generality that our instance includes a replenishing arm, i.e., that there exists $i_R \in [K]$ such that $b_{i_R} \in [1-\eps, 1]$. Note that this is indeed without loss of generality because of Lemmas~\ref{lem:balanced} and~\ref{lem:replenish-factor}. In this section, we prove the following guarantee regarding the regret incurred in the case of an \emph{unknown} replenishing arm.

\begin{theorem}\label{thm:regret-unknown}
Tuning $\delta = 2 \eps$, $M = K^2 \log (T) / \eps^2$ and 
\[
\eps = \left( \frac{K \cdot \log (T) \cdot \log (\lambda)}{T \cdot \log (1 - \lambda)}\right)^{1/3}
\]
Algorithm~\ref{algo:ETC-unknown-iR} incurs regret $\Regret(T) = \calO \left( \left( \frac{K \log (T) \log (\lambda)}{\log (1 - \lambda)} \right)^{1/3} T^{2/3} \right)$.
\end{theorem}

Let us define $\barb$ to be $\barb = \sum_{i \in [K]} b_i / K$. Based on Lemma~\ref{lem:replenish-factor}, and the fact that $b_i \geq 0, \forall i \in [K]$, it holds that $\barb \geq (1 - \eps) / K$. This will be useful in our analysis below. 

We first present the algorithm that achieves the desired regret guarantee for the case of an unknown replenishing arm. 

\begin{algorithm}[htbp]
\SetAlgorithmName{Algorithm}{}{}
\caption{$\BDES$ general $\lambda$, unknown $i_R$}\label{algo:ETC-unknown-iR}
\DontPrintSemicolon
\LinesNumbered
\SetAlgoNoEnd
Set $\eps, \delta, M$ as stated in Theorem~\ref{thm:regret-unknown}.\;
Initialize rounds $t=1$. \;
\tcc{Explore IV rewards and build their estimators: $\{\hr_i\}_{i \in [K]}$}
\For{arm $i \in [K]$}{
Initialize reward estimate $\hr_i = 0$. \;
\For(\tcp*[h]{Restore the state to at least $b_z-\eps$}){blocks $j \in [M]$}{ Choose an arm $z \in [K]$ uniformly at random.  \tcp*{$z$ = benchmark arm for state.}
\For{pulls $1, \dots, N(\lambda)$}{ 
Play arm $z$. \;
Update $t \gets t+1$.
}
Play arm $i$, observe reward $R_j^i$, and update: $\hr_i \gets \hr_i + \frac{R_j^i}{M}$. \tcp*{Play $i$ when $q\approx b_z-\eps$.}
Update $t \gets t+1$.
}
}
\tcc{Explore ES and build estimators: $\{\hb_i\}_{i \in [K]}$}
\For{arm $i \in [K]$}{
Initialize state estimator $\hv_i = 0$. \;
\For{pulls $1, \dots, N(\lambda)$}{
Play arm $i$. \;
Update $t \gets t+1$. \;
}
\For{blocks $j \in [M]$}{
Play arm $i$, observe reward $S_j^i$, and update: $\hv_i \gets \hv_i + \frac{S_j^i}{M}$. \tcp*{Play $i$ when $q \approx b_i$}
}
Compute baseline reward estimator: $\hb_i = \nicefrac{\hv_i}{\hr_i}$.\;
}
Play arm $i_R$ for $N_R$ rounds, updating $t \gets t+1$ after each one.\tcp*{Restore state to at least $1 - \eps$}
Feed $( \hr_i, \hb_i )$ in the Dynamic Programming algorithm and play the solution until the end of horizon $T$.
\end{algorithm}%

Our analysis follows a similar route as for the case of Theorem~\ref{thm:regret-ETC}. Importantly, Lemma~\ref{lem:hv-estimator} remains unchanged and still holds verbatim. What changes is the lemma with the estimator $\hr_i, \forall i \in [K]$ because now we have sampled uniformly at random a benchmark arm, rather than using the known replenishing arm.

\begin{lemma}\label{lem:hr-estimator-app}
Let $\bbar = \frac{1}{K} \sum_{i \in [K]} b_i$. Then, for the IV reward estimator of each arm in Line 10 of Algorithm~\ref{algo:ETC-unknown-iR} and any scalar $\delta$, it holds that: \[\Pr \left[ \left| \hr_i- \bbar \cdot r_i \right| \geq \delta \right] \leq 2 \exp \left( -2 M \cdot \left( \delta - \eps\right)^2 \right),\] 
\end{lemma}

\begin{proof}
From Hoeffding's inequality on $\hr_i$, we have that:
\begin{equation}\label{eq:hoeff-hri}
    \Pr \left[ \left| \hr_i - \E \left[ \hr_i \right] \right| \geq \delta \right] \leq 2 \exp \left( -2 M \delta^2 \right)
\end{equation}

From Lemma~\ref{lem:state-approx-bi}, regardless of the starting state and the prior history, if an arm $z$ is played repeatedly for $N(\lambda)$ rounds, then at round $t_{j}^i$ the system's state is at $q_{t_{j}^i} \geq b_z - \eps$. So (by definition of our setting) and conditioning on event $\calE_z = \{\text{arm $z$ is chosen as benchmark}\}$ the expected reward at the right next round (i.e., Line 10 of Algorithm~\ref{algo:ETC-unknown-iR}) is 
\[
\E \left[ R_j^i | \calE_z \right] = q_{t_j^i} \cdot r_i \in [(b_z - \eps) \cdot r_i, (b_z + \eps) \cdot r_i ],
\] 
This means that in expectation over the choice of $z$ (which happens uniformly at random) we have: 
\[ \E \left[ R_j^i \right] = \E \left[q_{t_j^i}\right] \cdot r_i \in [(\bbar - \eps) \cdot r_i, (\bbar + \eps) \cdot r_i ]
\]
 As a result, by the linearity of expectation and using the definition of $\hr_i$: 
\begin{align*}
\E \left[ \hr_i \right] &= \frac{\E \left[ R_j^i \right]}{M} = 
\frac{\sum_{j \in [M]} q_{t_j^i} \cdot r_i}{M} = r_i \cdot \frac{\sum_{j \in [M]} q_{t_j^i}}{M} \Rightarrow \E \left[ \hr_i \right] \in \left[ r_i\cdot (b_z- \eps), r_i \cdot (b_z + \eps) \right]
\end{align*}
From Equation~\eqref{eq:hoeff-hri}, we have that: 
\begin{align*}
    2 \exp \left( - 2 M \delta^2 \right) &\geq \Pr \left[ \hr_i - \E [\hr_i ] \geq \delta \; \text{or} \; \hr_i - \E[\hr_i] \leq - \delta \right] \\
    &\geq \Pr \left[ \hr_i \geq \bbar \cdot r_i + \eps + \delta \; \text{or} \; \hr_i - \E[\hr_i] \leq - \delta  \right] &\tag{$\E[\hr_i] \leq r_i \cdot b_z+ \eps$}\\
    &\geq \Pr \left[ \hr_i \geq \bbar r_i + \eps + \delta \; \text{or} \; \hr_i \leq \bbar r_i - \eps - \delta \right] &\tag{$\E[\hr_i] \geq b_z \cdot r_i  - \eps$} \\ 
    &= \Pr \left[ \left| \hr_i - \bbar r_i \right| \geq \delta + \eps  \right]
\end{align*}
Using as $\delta' = \delta+\eps$ in the latter gives the result.
\end{proof}
Next, we show that the $\hb_i$ estimators that are built from the second part of the algorithm are good estimators, despite no assumptions on $i_R$.
\begin{lemma}\label{lem:hb-estimator-app}
Let $\bbar = \sum_z b_z / K$. Then, for the baseline reward estimators of each arm $i$ in Line 19 of Algorithm~\ref{algo:ETC-unknown-iR} and any scalar $\delta \geq 2 \eps$, it holds that:
\[\Pr \left[ \left| \hb_i - \frac{b_i}{\bbar}\right| \geq \delta \right] \leq 8 \exp \left( -2 M \cdot (\eps - \delta)^2 \right)
\]
\end{lemma}

\begin{proof}
We follow the steps of the proof of Lemma~\ref{lem:hb-estimator}.
Fix an arm $i \in [K]$ and let us use $e_v$ and $e_r$ to denote the following quantities: $e_v = \hv_i - v_i$ and $e_r = \hr_i - \bbar \cdot r_i$ respectively. Then, we have that:
\begin{align*}
   \Pr \left[ \left| \frac{\hat{v}_i}{\hat{r}_i} - \frac{v_i}{\bbar \cdot r_i} \right| \geq \delta \right] 
   &=\Pr \left[ \left| \frac{v_i+e_v}{\bbar \cdot r_i+e_r} - \frac{v_i}{\bbar \cdot r_i} \right| \geq \delta \right] \\
   &=\Pr \left[ \left| \frac{\bbar r_i e_v - e_r v_i}{\bbar r_i (\bbar r_i + e_r)} \right| \geq \delta \right] \\
   &\leq \Pr \left[ \left| \frac{e_v}{\bbar r_i+e_r} \right| + \left| b_i \frac{e_r}{\bbar r_i + e_r} \right|\geq \delta \right] \\ 
   &\leq \underbrace{\Pr \left[ \left| \frac{e_v}{\bbar r_i+e_r} \right| \geq \delta/2 \right]}_{Q_1} + \underbrace{\Pr \left[ b_i \cdot \left| \frac{e_r}{\bbar \cdot (\bbar r_i+e_r)} \right| \geq \delta/2 \right]}_{Q_2} \numberthis{\label{eq:hoeff-bhat-app}}
\end{align*}
where the first inequality is due to the triangle inequality and the fact that $\Pr[a < c] \leq \Pr[b < c]$ for $a \leq b$, and the second inequality is due to the fact that when $a + b \geq c$, then $\Pr[a+b \geq c] \leq \Pr[a \geq c/2] + \Pr[b \geq c/2]$.

To upper bound $Q_1$ and $Q_2$, we condition on the following event: $\calE_i' = \{ |e_r| \leq \delta \}$. Note that the probability with which the complement $\calE_i$ happens is given by Lemma~\ref{lem:hr-estimator-app} and is: 
\begin{equation}\label{eq:calE-cond-app}
    \Pr [\calE_i] \geq 2 \exp \left(-2 M \cdot (\delta - \eps)^2 \right) 
\end{equation}
Rewriting $Q_1$:
\begin{align*}
    Q_1 = \Pr \left[ |e_v| \geq \frac{\delta}{2} \cdot |\bbar r_i + e_r|\right] \leq \Pr \left[ |e_v| \geq \frac{\delta}{2} \cdot \Big||\bbar r_i| - |e_r| \Big|\right] \numberthis{\label{eq:Q1-relax-app}}
\end{align*}
Conditioning on $\calE_i'$ we get: 
\begin{align*}
    \Pr \left[ |e_v| \geq \frac{\delta}{2} \cdot \Big | |\bbar r_i| - |e_r| \Big | \, \Big | \, \calE_i' \right] &\leq \Pr \left[ |e_v| \geq \frac{\delta}{2} \cdot | \bbar r_i - \delta | \right] \\
    &\leq 2 \exp \left( -2M \cdot \left(\frac{\delta}{2} \cdot |\bbar r_i - \delta| - \eps \right)^2 \right) &\tag{Lemma~\ref{lem:hv-estimator}}\\
    &\leq 2 \exp \left( -2 M \cdot (\eps^2 - \eps \delta) \right) \numberthis{\label{eq:Q1-cond-app}}
\end{align*}
where the last inequality is due to the fact that $|\bbar r_i - \delta| \leq 1$. From the law of total probability:
\begin{align*}
    Q_1 &= \Pr \left[ |e_v| \geq \frac{\delta}{2} \cdot |\bbar r_i + e_r| \, \Big | \, \calE_i'\right] \cdot \Pr \left[ \calE_i' \right] + \Pr \left[ |e_v| \geq \frac{\delta}{2} \cdot |\bbar r_i + e_r| \, \Big | \, \calE_i\right] \cdot \Pr \left[ \calE_i \right] \\
    &\leq \Pr \left[ |e_v| \geq \frac{\delta}{2} \cdot \Big | |\bbar r_i| - |e_r| \Big | \, \Big | \, \calE_i'\right] \cdot \Pr \left[ \calE_i' \right] + \Pr \left[ |e_v| \geq \frac{\delta}{2} \cdot \Big | |\bbar r_i| - |e_r| \Big |\, \Big | \, \calE_i\right] \cdot \Pr \left[ \calE_i \right] \\
    &\leq 2 \exp \left( M \cdot \left( \eps^2 - \delta \right) \right) \cdot 1 + 1 \cdot 2 \exp \left( - 2M \cdot (\delta - \eps )^2 \right)
\end{align*}
where the first inequality is due to Eq.~\eqref{eq:Q1-relax-app} and the last one is due to Eqs.~\eqref{eq:calE-cond-app},~\eqref{eq:Q1-cond-app}.

We now turn our attention to $Q_2$: 
\begin{align*}
    Q_2 \leq \Pr \left[ |e_r| \geq \frac{\delta}{2(K - \eps)} \cdot \left| \bbar r_i + e_r \right| \right] \leq \Pr \left[ |e_r| \geq \frac{\delta}{2K} \cdot \left| \bbar r_i + e_r \right| \right]
\end{align*}
where the first inequality is due to the fact that $\bbar \geq 1/K$. Using exactly the same reasoning as above, but now coupled with Lemma~\ref{lem:hr-estimator-app} instead of Lemma~\ref{lem:hv-estimator} we have that: 
\[
Q_2 \leq 2 \exp \left( M \cdot \left( \eps^2 - \eps \delta / K \right) \right)+ 2 \exp \left( - 2M \cdot (\delta/ K  - \eps )^2 \right)
\]
Adding the two upper bounds from $Q_1$ and $Q_2$ to Equation~\eqref{eq:hoeff-bhat-app} we get the stated result.
\end{proof}

We are now ready to prove Theorem~\ref{thm:regret-unknown}.

\begin{proof}[Proof of Theorem~\ref{thm:regret-unknown}]
The proof follows directly the proof of Theorem~\ref{thm:regret-ETC} but we use the Lemmas that we stated above, for the estimators computed by Algorithm~\ref{algo:ETC-unknown-iR}.
\end{proof}

\section{Appendix for Section~\ref{sec:small-lambda}} 

\subsection{EXP3.P}\label{app:exp3.p}

\begin{algorithm}[htbp]
\SetAlgorithmName{Algorithm}{}{}
\caption{ EXP3.P Algorithm }\label{algo:exp3.p}
\DontPrintSemicolon
\LinesNumbered
\SetAlgoNoEnd
\textbf{Input.} $ \eta = 0.95 \sqrt{\frac{ \log K }{KT}}$, $\gamma = 1.05 \sqrt{\frac{K \log K }{T}}, \beta =  \sqrt{\frac{ \log (K \delta ^{-1})}{KT}}, \forall \delta \in (0,1)$ \;
Initially at time $t=1$, let $p_1$ be the uniform distribution over $[K]$. \;
\For{$t \in [T]$}{
Choose an arm $I_t \in [K]$ from probability distribution $p_t$. \;
Observe reward $g_{I_t,t} \sim \Bern(r_{I_t} q_t)$. \;
For each arm $i \in [K]$ compute the estimated biased gain:
$$ \tilde g_{i,t} = \frac{ g_{i,t} {\mathbbm{1}} \{I_t = i\} + \beta }{p_{i,t}} $$
and update the estimated cumulative gain: $\tilde G_{i,t} = \sum_{s \in [t]} \tilde g_{i,s}.$ \;
Compute the new probability distribution over arms $p_{i,t+1}$:
$$ p_{i,t+1} = (1 - \gamma) \frac{\exp (\eta \tilde G_{i,t})}{\sum_{k \in [K]} \eta \tilde G_{k,t}} + \frac{\gamma}{K}.$$
}
\end{algorithm}%

\section{Appendix for Section~\ref{sec:sticky-arms}}\label{app:sticky-arms}

\begin{proof}[Proof of Theorem~\ref{thm:regret-sticky} for $\lambda =1$]
We first list a property that is very useful for our proof. Note that the average of rewards observed in each group of size $2$ containing arms $(i, j)$ satisfies $\E [ (\tilde r_t + \tilde r_{t+1})/2] = (r_i b_j + r_j b_i)/2$, since we have i.i.d. and $\sigma/\sqrt{2}$-subgaussian observations.

Let $(\ist \diamond \jst)$ be the optimal meta-arm. Let $\Delta_{(i \diamond j)}$ be the gap of meta-arm $(i \diamond j)$, defined as:
\begin{equation*}
    \Delta_{(i \diamond j)} = \frac{r_{\ist} b_{\jst} + r_{\jst} b_{\ist} - r_{i} b_{j} - r_{j} b_{i}}{2}.
\end{equation*}
Then, the regret incurred throughout $T$ rounds can be written as:
\begin{align*}
    \regDES(T) &= \sum_{t=1}^T \left(\frac{r_{\ist} b_{\jst} + r_{\jst} b_{\ist}}{2} - r_{I_t} b_{I_{t-1}}\right) \\
    &\leq \sum_{1 \leq i \leq j \leq K} \Delta_{(i \diamond j)} N_{(i \diamond j)} + \sum_{t \in [T]} \mathbb{I}[\text{transition between two meta-arms happens at $t$}]
\end{align*}
where $N_{(i \diamond j)}$ is the number of pulls of meta-arm $(i \diamond j)$ during $T$ rounds. Since $r,b \in [0,1]$, then the second term in the above is upper bounded by $B K$ as in each batch the transition happens only between active arms. As a result, the regret is upper bounded by:
\begin{equation}\label{eq:regr-sticky-app}
    \regDES(T) \leq 2 \sum_{1 \leq i \leq j \leq K} \Delta_{(i \diamond j)} N_{(i \diamond j)} + BK \,.
\end{equation}
Next, we bound $N_{(i \diamond j)}$ using variations of standard arm-elimination techniques. We call the estimation for a meta-arm $(i \diamond j)$ at the end of batch $\beta$, $\delta$-correct, if the true mean of that meta-arm is within $\sqrt{2 \log(1/\delta)/ c_\beta}$ of estimated value, i.e.,
\begin{equation*}
    \left|\hat{\mu}_{(i \diamond j)} - \frac{r_i b_j + r_j b_i}{2} \right| \leq \sqrt{\frac{2 \log(1/\delta)}{c_\beta}}\,.
\end{equation*}
Now as $\hat{\mu}_{(i \diamond j)}$ contains of $c_\beta$ i.i.d. samples with mean $\mu_{(i \diamond j)} = (r_i b_j + r_j b_i)/2$ (standard deviation at most $1$), Hoeffding's inequality implies that each active meta-arm is $\delta$-correct with probability at least $1-\delta$. Since we have $K(K+1)/2$ meta-arms and $B$ batches, then selecting $\delta = 1/(2K^2BT)$ and a union bound implies that with probability $1-1/T$, all active meta-arms are $\delta$ valid in all batches.

Now if this happens, it basically means that all active arms $(i, j)$ at the end of every batch satisfy
\begin{equation*}
    \left|\hat{\mu}_{(i \diamond j)} - \frac{r_i b_j + r_j b_i}{2} \right| \leq \sqrt{\frac{2 \log(2K^2 B T)}{c_\beta}}\,.
\end{equation*}
This also means that the best meta-arm $(\ist \diamond \jst)$ is never eliminated. We can now derive an upper bound on the number of pulls of each of these sub-optimal $(i \diamond j)$ meta-arms as follows. Let $\beta+1$ be the last batch in which arm $(i \diamond j)$ was active. Since this arm was not eliminated at batch $\beta$, we have
\begin{equation*}
    \Delta_{(i \diamond j)} \leq 2 \sqrt{\frac{2 \log(2K^2 B T)}{ c_\beta}}\,,
\end{equation*}
which after re-arrangement means that $c_\beta \leq 8 \log \left(2K^2 BT\right) \Delta_{(i \diamond j)}^{-2}$. Note that this also means that
\begin{equation*}
    N_{(i \diamond j)} \leq c_{\beta+1} =  w + w c_\beta=  w + 8 w \log\left(2K^2 BT\right) \Delta_{(i \diamond j)}^{-2}  \,.
\end{equation*}
Putting everything together:
\begin{align*} 
\regDES(T) &\leq \sum_{1 \leq i \leq j \leq K} w \Delta_{(i \diamond j)}  + 2 w c_{\beta} \sqrt{\frac{2 \log (2K^2 BT)}{c_{\beta}}} + BK \\
&= \sum_{1 \leq i \leq j \leq K} w \Delta_{(i \diamond j)}  + 2 w \sqrt{c_{\beta}} \sqrt{2 \log (2K^2 BT)} + BK
\end{align*}
By Jensen's inequality for concave function $f(x) = \sqrt{x}$ we get:
$$ \frac{1}{K(K+1)/2} \sum_{1 \leq i \leq j \leq K} \sqrt{c_{\beta}} \leq \sqrt{\frac{1}{K(K+1)/2} \sum_{1 \leq i \leq j \leq K} c_{\beta}} \leq \sqrt{\frac{2T}{K(K+1)/2}}. $$
Plugging this to regret and replacing $w = T^{1/B}$ we get:
\begin{align*} 
\regDES(T) &\leq 2 T^{1/B} \sqrt{2 \log (2K^2 BT)\frac{K(K+1)}{2} T} + KB = \calO \left( K \sqrt{T \log \left( 2 K^2 T \right)} \right). \,
\end{align*}
\end{proof}

\section{Appendix for Section~\ref{sec:robustness}}\label{app:robustness}

\subsection{Noise-Perturbed Model}

\begin{proof}[Proof of \cref{lem:noise-perturbed}]
The proof is split into $3$ parts depending on the region of $\lambda$ that we focus on. Before we delve into these parts, note that from the definition of a $\sigma$-subGaussian, we have that $\Pr [ |\nu_t| \geq \tau] \leq \exp(-\sigma^2 / t^2)$. In other words, $\Pr [|\nu_t| \geq \sigma \sqrt{\log (T / \delta)}] \leq \delta/T$. By taking a union bound on all rounds $t$ we get:
\begin{equation}\label{eq:hoeff-noise}
\Pr[ \forall \ t \ : |\nu_t| \geq \sigma \sqrt{\log (T/ \delta)} ] \leq \delta. 
\end{equation}

\xhdr{Part I: $\lambda \in [0, \widetilde{\Theta}(1/T)]$.}

Let us denote by $Rew(\textsf{EXP3.P})$ the reward collected by running \textsf{EXP3.P}, i.e., $Rew(\textsf{EXP3.P}) = \sum_{t \in [T]} r_{I_t} q_t$, where the sequence of chosen arms $I_t$ and the state $q_t$ depend on \textsf{EXP3.P}. Let also $\bar{q}_t$ be the sequence of states induced by policy $\tilde{\pi}^{\star}$. Then, for the regret in the noise-perturbed model we have: 
\begin{align*}
    \regDES(T) &= \E \left[ \max_{\tilde{\pi}^{\star}} \sum_{t \in [T]} r_{\tilde{\pi}^{\star}_t} (\bar{q}_t + \nu_t) - Rew(\textsf{EXP3.P})\right] \\
    &= \underbrace{\E \left[ \max_{\tilde{\pi}^{\star}} \sum_{t \in [T]} r_{\tilde{\pi}^{\star}_t} (\bar{q}_t + \nu_t) - \sum_{t \in [T]} \widetilde{r}_{I^{\star},t}\right]}_{Q} + \underbrace{\E \left[\sum_{t \in [T]} \widetilde{r}_{I^{\star},t} - Rew(\textsf{EXP3.P})\right]}_{R_{\text{EXT}}(T)} \numberthis{\label{eq:noisy-small}}
\end{align*}
where (following the notation of Section~\ref{sec:small-lambda}) we use $\sum_{t \in [T]} \widetilde{r}_{I^{\star},t}$ to denote the benchmark of \textsf{EXP3.P}. Next, we focus on upper bounding term $Q$. 

Following the steps from \cref{sec:small-lambda}, we add and subtract $\E \left[ \max_{\pi^{\star}} \sum_{t \in [T]}q_t^{\pi^{\star}} r_{\pi^\star_t} \right]$ (i.e., the benchmark reward in hindsight had the states \emph{not} been noisily perturbed) from $Q$. From the analysis of \cref{sec:small-lambda} note that $\E \left[ \max_{\pi^{\star}} \sum_{t \in [T]}q_t^{\pi^{\star}} r_{\pi^\star_t} \right] - \sum_{t \in [T]}\tilde{r}_{I^\star,t}$ corresponds to $-A$ and is upper bounded by $1 - (1 - \lambda)^T \cdot \OPT$. Putting everything together, we can get the following upper bound for $Q$: 
\begin{align*}
    Q &\leq \E \left[ \max_{\tilde{\pi}^{\star}} \sum_{t \in [T]} r_{\tilde{\pi}^{\star}_t} (\bar{q}_t + \nu_t) - \max_{\pi^{\star}}\sum_{t \in [T]} r_{\pi^{\star}_t} q^{\pi^\star}_t \right] - A\\
    &\leq \E \left[ \max_{\tilde{\pi}^\star} \sum_{t \in [T]} r_{\tilde{\pi}^{\star}_t} \nu_t \right] + (1 - (1 - \lambda)^T) \cdot \OPT \leq \sigma \log (T/\delta) T + \delta T + (1 - (1 - \lambda)^T)\cdot \OPT \numberthis{\label{eq:small-lambda-almost-done}}
\end{align*}
where the second inequality is because of the fact that the benchmark for $\pi^{\star}$ maximizes the state-augmented reward for the noiseless model and the last inequality uses \cref{eq:hoeff-noise}. Using \cref{eq:noisy-small} and \cref{eq:small-lambda-almost-done} with the bound for $R_{\text{EXT}}(T)$ for $\textsf{EXP3.P}$, we get the result.

\xhdr{Part II: $\lambda \in (\widetilde{\Theta}(1/T), \widetilde{\Theta}(1 - 1/\sqrt{T}))$.}

Note, that even in the noise-perturbed model the $q_t$ part of the state is still defined deterministically and Lemma~\ref{lem:state-approx-bi} (and Corollary~\ref{lem:replenish-time}) still hold. Thus, after playing an arm $i$ repeatedly for $N(\lambda)$ rounds we get $|q_t - b_i| \leq \eps$. Algorithm~\ref{algo:ETC-known-iR} plays replenishing arm $i_R$ for $N(\lambda)$ rounds and so $ q_t \geq 1 - \eps$. After the state has converged to $b_{i_R}$, the algorithm starts building the estimators. While noise does not affect the $q_t$ part of the state, it \emph{does} affect $\tilde{q}_t$ which in turn is where the estimators are built from. More specifically, for the $\hat{r}_i$ estimators, Algorithm~\ref{algo:ETC-known-iR} after playing $i_R$ for $N(\lambda)$ rounds, takes a sample of arm $i$; hence, the sample came from $\Bern (r_i \cdot (1 - \eps + \nu_t))$. As a result, Lemmas~\ref{lem:rew-estimator1} and~\ref{lem:hb-estimator} change as follows: 

\begin{lemma}
    In the noise-perturbed model, the estimators $\hv_i$ and $\hb_i$ satisfy the following respectively:
    $$ \Pr [|\hr_i - r_i| \geq 2\delta]  \leq 4\exp(-2M\delta^2)  $$
    and 
    $$ \Pr[| \hb_i - b_i| \geq 2\delta] \leq 4 \exp( - 2 M(\delta - \eps)^2) + 4 \exp(-2 M (\eps^2 - \eps \delta).$$
\end{lemma}

\begin{proof}
    From Hoeffding's inequality on $\hat{\nu}_t = \sum_{t \in [M]} \frac{\nu_t}{M}$ using that the block size is $M$ rounds, we get (since $\nu_t \leq 1$ and $\E [\nu_t] = 0$):
    $$ \Pr [|\hat{\nu}_t| \geq \delta]  \leq 2\exp(-2M\delta^2)  $$
    Since $\hat{r_i} = \sum_{t \in [M]} \frac{r_i \cdot \tilde{q}_t}{M} = \frac{r_i \cdot (q_t + \nu_t) }{M} $ using the same steps as in Lemma\ref{lem:rew-estimator1} and a union bound we get:
     $$ \Pr[| \hat{r}_i - r_i| \geq 2\delta] \leq 2 \exp( - 2 M(\delta - \eps)^2) + 2 \exp(-2 M \delta^2) \leq 4 \exp( - 2 M(\delta - \eps)^2) $$
     The proof for the $\hb_i$ estimator is almost identical.
\end{proof}
Using this lemma, we can conclude the proof for the regret of Algorithm~\ref{algo:ETC-known-iR}. Let $Rew(\ALG) = \sum_{t \in [T]} r_{I_t} q_t$. To distinguish between the states induced by sequences $\tilde{\pi}^{\star}$ and $\pi^{\star}$ we use $\{\bar{q}_t\}_t$ and $q^{\pi^\star}_t$ respectively. Then:
    \begin{align*}
        R_{\DES} (T) &=  \E \left[ \max_{\tilde{\pi}^{\star} }\sum_{t \in [T]} r_{\tilde{\pi}^{\star}_t } (\bar{q}_t + \nu_t)  - Rew(\ALG)\right] \\
        & = \E \left[ \max_{\tilde{\pi}^{\star} }\sum_{t \in [T]} r_{\tilde{\pi}^{\star}_t } (\bar{q}_t + \nu_t) - \max_{\pi^{\star}} \sum_{t \in [T]} r_{\pi_{t}^{\star}} q^{\pi^\star}_t + \max_{\pi^{\star}} \sum_{t \in [T]} r_{\pi_{t}^{\star}} q^{\pi^\star}_t - Rew(\ALG)\right] \\
        &\leq  \E \left[\max_{\tilde{\pi}^{\star}}\sum_{t \in [T]} r_{\tilde{\pi}^{\star}_t } (\bar{q}_{t} + \nu_t) - \max_{\pi^{\star}} \sum_{t \in [T]} r_{\pi_{t}^{\star}} q^{\pi^\star}_t \right] + \tcalO \left( \left(\frac{K \log (T) \log (\lambda)}{\log (1 - \lambda)}\right)^{1/3} \cdot T^{2/3}\right) \\
        &\leq \E \left[\max_{\tilde{\pi}^{\star} }\sum_{t \in [T]} r_{\tilde{\pi}^{\star}_t } \bar{q}_{t} +
        \max_{\tilde{\pi}^{\star} }\sum_{t \in [T]} r_{\tilde{\pi}^{\star}_t }  \nu_t- \max_{\pi^{\star}} \sum_{t \in [T]} r_{\pi_{t}^{\star}} q^{\pi^\star}_t \right] + \tcalO \left( \left(\frac{K \log (T) \log (\lambda)}{\log (1 - \lambda)}\right)^{1/3} \cdot T^{2/3}\right) \\
        &\leq \E \left[
        \max_{\tilde{\pi}^{\star} }\sum_{t \in [T]} r_{\tilde{\pi}^{\star}_t }  \nu_t \right] + \tcalO \left( \left(\frac{K \log (T) \log (\lambda)}{\log (1 - \lambda)}\right)^{1/3} \cdot T^{2/3}\right) \\
        &= (1 - \delta) \sigma \log(T / \delta) T + \delta T + \tcalO \left( \left(\frac{K \log (T) \log (\lambda)}{\log (1 - \lambda)}\right)^{1/3} \cdot T^{2/3}\right) \\
        &\leq \tcalO \left( \left(\frac{K \log (T) \log (\lambda)}{\log (1 - \lambda)}\right)^{1/3} \cdot T^{2/3} + \sigma T  \right) 
    \end{align*}

\xhdr{Part III: $\lambda \in [\widetilde{\Theta}(1 - 1/\sqrt{T}), 1].$}

The intuition and analysis of this case bears similarities with the analysis for the case where $\lambda = 1 - \epsilon$ (Section~\ref{sec:sticky-arms}). 

If the noise added to a round $\nu_t$ is ``small enough'' (specifically, if $\nu_t \leq \sigma \sqrt{\log (T / \delta)}$ for $\sigma < \sqrt{T}$), then the analysis is identical to the case where $\lambda = 1 - \epsilon$. If the noise added to a round $\nu_t$ is greater than $\sigma \sqrt{\log (T / \delta)}$, then in the worst case the expected reward can be affected by a $\sigma T$ factor in total. Putting everything together (and using the Hoeffding bound of \cref{eq:hoeff-noise}) we get the result. 
\end{proof}

\subsection{Unknown $\lambda$}

In this section, we present the analysis when $\lambda$ is unknown. We first present the \emph {Unknown General $\lambda$} which will be used in the main algorithm of the unknown $\lambda$. In Algorithm~\ref{algo:gen-unknown-lambda} we get as an input a $\tilde N(\lambda)$ and two arms $i,j$. Then, we use a technique from Algorithm~\ref{algo:ETC-known-iR} to learn $r_i b_i$ and $r_i q_{ALT_{i,j}}$, which will defined below. After that, we solve an equation and if $\tilde N(\lambda)$ is near the real $N(\lambda)$ we can get a good estimator $\hat{\lambda}$ for $\lambda$. 

\begin{algorithm}[htbp]
\caption{ $\BDES$ Unknown General $\lambda$}\label{algo:gen-unknown-lambda}
\DontPrintSemicolon
\LinesNumbered
\SetAlgoNoEnd
\textbf{Input.} $\tilde N(\lambda)$, arm $i$, arm $j$. \;
Choose $\eps$, $\delta$, $M$ based as in $\tilde N(\lambda)$ and Theorem~\ref{thm:regret-ETC}\;
\tcc{Build estimators for $r_ib_j$}
\For{blocks $\in [M]$}{
    \For{pulls $\in [\tilde N(\lambda)]$}{
        Play arm $j$. \tcp*{$q_t$ near $b_j$}
    }
    Play arm $i$, observe reward $R_{i,b_j}$ and update: $\hat{r}_{i,b_j} = \hat{r}_{i,b_j} + R_{i,b_j}/M$. \tcp*{take a sample of $r_i \cdot b_j$}
}
\tcc{Build estimators for $r_ib_i$}
\For{blocks $\in [M]$}{
    \For{pulls $\in [\tilde N(\lambda)]$}{
        Play arm $j$. \tcp*{$q_t$ near $b_j$}
    }
    \For{pulls $\in [\tilde N(\lambda)]$}{
        Play arm $i$. \tcp*{$q_t$ near $b_i$}
    }
    Play arm $i$, observe reward $R_{i,i}$ and update: $\hat{r}_{i,i} = \hat{r}_{i,i} + R_{i,i}/M$. \tcp*{take a sample of $r_i \cdot b_i$}
}

\tcc{Build estimators for $\frac{r_ib_j - (1- \lambda) r_ib_i}{2 - \lambda}$}
\For{blocks $\in [M]$}{
    \For{pulls $\in [\tilde N(\lambda)]$}{
        Play arm $i$ \;
    }
    \For{pulls $\in [\tilde N(\lambda)]$}{
        Play arm $i$ \;
        Play arm $j$ \;
    }
    Play arm $i$, observe reward $R_{i,j}$ and update: $\hat{r}_{i,j} = \hat{r}_{i,j} + R_{i,j}/M$. \tcp*{take a sample of $r_i \cdot \frac{b_j + (1- \lambda)}{2- \lambda}$}
}
Solve $\hat{\lambda} = \frac{\hat{r}_{i,i} + \hat{r}_{i,j} - 2 \hat{r}_{i,j}}{\hat{r}_{i,i} - \hat{r}_{i,j}}$ \;
\textbf{Output.} $\hat{\lambda}$.
\end{algorithm}%

Formally, in this section we will be proving the following.

\begin{theorem}
    Algorithm~\ref{algo:unknown-lambda} achieves regret:
     \begin{equation*}
        \regDES (T) = 
        \begin{cases}
          \tcalO \left( K^{1/3}  T^{2/3} \right) & \textrm{for } \lambda \in [0, \Theta(1/T^2)]\\
          \calO \left( T^{b/a} \right) & \textrm{for } \lambda = T^{-a/b} \textrm{ and } \max_{i,j \in [K]}\{ r_i |b_i - b_j| \} > \frac{\sqrt{\log T} K^{1/3}}{T^{1/3}} \\
          (1- 1/e) \OPT & \textrm{for } \lambda = \Theta(1/T) .\textrm{ and } \max_{i,j \in [K]}\{ r_i |b_i - b_j| \} > \frac{\sqrt{\log T} K^{1/3}}{T^{1/3}} \\
          \tcalO \left( K^{1/3}  T^{2/3} \right) & \textrm{for } \lambda \in (\Theta(1/T^2), 1] \textrm{ and } \max_{i,j \in [K]}\{ r_i |b_i - b_j| \} \leq \frac{\sqrt{\log T} K^{1/3}}{T^{1/3}}\\
           \calO\left( \left(\frac{K \log (T) \log (\lambda)}{\log (1 - \lambda)}\right)^{1/3} \cdot T^{2/3}\right) & \textrm{for } \lambda \in (O(K/T)^{1/3}, 1] \textrm{ and } \max_{i,j \in [K]}\{ r_i |b_i - b_j| \} > \frac{\sqrt{\log T} K^{1/3}}{T^{1/3}}\\
        \end{cases}
    \end{equation*}
\end{theorem}

The next lemma states that irrespective of the state where you start from, if you alternate between two fixed arms for many rounds, then the state converges to a closed form solution that involves $\lambda$ and the $b_i$'s of the two alternating arms.

\begin{lemma} \label{lem:alt-2-arms}
    Fix two arms $i,j \in [K]$ with $i \neq j$, $\lambda> 0$ and a scalar $\eps>0$. Assume that at some round $s$, after a history of play $H'$, we are at state $q_s$. Then, playing alternately arms $i,j$ for infinitely many rounds $t$ makes the state become:
    \begin{equation*}
        q_{s+t+1} = \begin{cases}
            b_i \frac{1 - \lambda}{2 - \lambda} + b_j \frac{1}{2 - \lambda} & \textrm{if} \ t \equiv 0 \pmod 2\\
            b_j \frac{1 - \lambda}{2 - \lambda} + b_i \frac{1}{2 - \lambda} & \textrm{else} \\
        \end{cases}.
    \end{equation*}
\end{lemma}

\begin{proof}
    Let $\textrm{ALT}_{i,j}$ be the algorithm that continuously alternates between arm $i$ and $j$.
    We first prove that :
    \begin{equation}\label{eq:state-alt}
         q_{s+t+1} \left( H^{ALT_{i,j}}_{s:s+t} \right) = \begin{cases}
              (1- \lambda)^{t} q_{s+1} + \lambda \sum_{\tau = 0 }^{(t-2)/2} (1 - \lambda)^{2\tau} b_j +  (1 - \lambda)^{2\tau + 1} b_i & \textrm{if } t \equiv 0 \pmod 2, \\
             (1- \lambda)^{t} q_{s+1} + \lambda \sum_{\tau = 0 }^{(t-3)/2} (1 - \lambda)^{2\tau + 1} b_j +  \sum_{\tau = 0 }^{(t-1)/2} (1 - \lambda)^{2\tau} b_i & \textrm{else}
         \end{cases},
    \end{equation}
    where $t \geq 0$, using induction.
     For the base case rounds $s+1$, $s+2$ the state becomes:
    \begin{equation*}
        q_{s+1}\left( H^{ALT_{i,j}}_{s:s} \right) = q_{s+1} \quad \textrm{and}  \quad  q_{s+2}\left( H^{ALT_{i,j}}_{s:s+1} \right) = (1 - \lambda) q_{s+1} + \lambda b_i. 
    \end{equation*}
    respectively. 
    For the inductive step, assume  w.l.o.g. $t = n \equiv 0 \pmod{2}$ and:
    \begin{equation*}
        q_{s+n+1} \left( H^{ALT_{i,j}}_{s:s+n} \right) = (1- \lambda)^n q_{s+1} + \lambda \sum_{\tau = 0 }^{(n-2)/2} (1 - \lambda)^{2\tau } b_j + (1 - \lambda)^{2\tau + 1} b_i,
    \end{equation*}
    Then, in round $s+n$ algorithm chooses arm $i$ and in round $s + n +1$ ($t = n+1$) the state becomes:
    \begin{align*}
        q_{s+n+2} \left( H^{ALT_{i,j}}_{s:s+n+1} \right) &= (1- \lambda) q_{s+n+1} + \lambda b_i \\
        & = (1-\lambda) \cdot \left( (1- \lambda)^n q_{s+1} + \lambda \sum_{\tau = 0 }^{(n-2)/2} (1 - \lambda)^{2\tau } b_j + (1 - \lambda)^{2\tau + 1} b_i \right) + \lambda b_i \\
        & =  (1-\lambda)^{n+1} q_{s+1} + \lambda \sum_{\tau = 0 }^{(n-2)/2} (1 - \lambda)^{2\tau +1 } b_j + (1 - \lambda)^{2\tau + 2} b_i  + \lambda b_i \\
        &= (1-\lambda)^{n+1} q_{s+1} + \lambda \sum_{\tau = 0 }^{(n-2)/2} (1 - \lambda)^{2\tau +1 } b_j + \lambda \sum_{\tau = 0}^{n/2} (1 - \lambda)^{2 \tau}b_i \\
        &= (1-\lambda)^{t} q_{s+1} + \lambda \sum_{\tau = 0 }^{(t-3)/2} (1 - \lambda)^{2\tau +1 } b_j + \lambda \sum_{\tau = 0}^{(t-1)/2} (1 - \lambda)^{2\tau}b_i
    \end{align*}
    and for $t = n+2$ (in round $ s + n+1$ the algorithm chooses arm $j$) :
    \begin{align*}
        q_{0:s+n+3} \left( H^{ALT_{i,j}}_{s:s+n+1} \right) &= (1- \lambda) q_{0:s+n+2} + \lambda b_j \\
        &= (1 - \lambda) \left( (1-\lambda)^{n+1} q_{s+1} + \sum_{\tau = 0 }^{(n-2)/2} (1 - \lambda)^{2\tau +1 } b_j + \lambda \sum_{\tau = 0}^{n/2} (1 - \lambda)^{2 \tau}b_i \right) + \lambda b_j \\
        & = (1 - \lambda)^{n+2} + \lambda \sum_{\tau = 0}^{n-2/2} (1 - \lambda)^{2\tau +2} b_j + \lambda b_j + \lambda \sum_{\tau = 0 }^{n/2} (1 - \lambda)^{2 \tau + 1} b_i \\
        &= (1 - \lambda)^{n+2} + \lambda \sum_{\tau = 0}^{n/2} (1 - \lambda)^{2\tau} b_j  + \lambda \sum_{\tau = 0 }^{n/2} (1 - \lambda)^{2 \tau + 1} b_i \\
        &  =  (1 - \lambda)^t + \lambda \sum_{\tau = 0}^{(t-2)/2} (1 - \lambda)^{2\tau} b_j + (1 - \lambda)^{2 \tau + 1} b_j
    \end{align*}
    which completes the proof of Eq.~\eqref{eq:state-alt}.
    Using geometric sums: 
    \begin{equation*}
        \lambda \sum_{\tau = 0}^{(t-2)/2} (1-\lambda)^{2 \tau} = \lambda \frac{1 - (1 - \lambda)^{t}}{1 - (1 - \lambda)^2}  = \frac{1 - (1 - \lambda)^{t}}{2 - \lambda}.
    \end{equation*}
    Putting in Eq.~\eqref{eq:state-alt} we get:
    \begin{equation*}
         q_{s+t+1} \left( H^{ALT_{i,j}}_{s:s+t} \right) = \begin{cases}
              (1- \lambda)^{t} q_{s+1} + \frac{1 - (1 - \lambda)^t}{2 - \lambda} b_j +  (1 - \lambda) \frac{1 - (1 - \lambda)^t}{2 - \lambda}  b_i & \textrm{if } t \equiv 0 \pmod 2, \\
             (1- \lambda)^{t} q_{s+1} + (1 - \lambda) \frac{1 - (1 - \lambda)^{t-2}}{2 - \lambda} b_j + \frac{1 - (1 - \lambda)^{t-2}}{2 - \lambda}  b_i & \textrm{else}
         \end{cases},
    \end{equation*}
    It is obvious that take the limit to infinity in Eq.~\ref{eq:state-alt}:
    \begin{equation*}
        \lim_{t \to \infty} q_{s+t} = \begin{cases}
            b_i \frac{1 - \lambda}{2 - \lambda} + b_j \frac{1}{2 - \lambda} & \textrm{if} \ t \equiv 0 \pmod 2 \\
            b_j \frac{1 - \lambda}{2 - \lambda} + b_i \frac{1}{2 - \lambda} & \textrm{else} \\
        \end{cases}.
    \end{equation*}
\end{proof}

Using the Lemma~\ref{lem:hb-estimator} we can play $N(\lambda)$ rounds two arms $i,j$ alternately and get state:
\begin{equation*}
    \left|q_{s+2N(\lambda)}  -
            b_i \frac{1 - \lambda}{2 - \lambda} - b_j \frac{1}{2 - \lambda} \right| \leq \eps. \quad  \textrm{if} \ t \equiv 0 \pmod 2 
\end{equation*}

\begin{equation*}
    \left| q_{s+2N(\lambda)}  -
            b_j \frac{1 - \lambda}{2 - \lambda} - b_i \frac{1}{2 - \lambda} \right| \leq \eps. \quad  \textrm{if} \ t \equiv 1 \pmod 2 
\end{equation*}

\begin{lemma} \label{lem:hat-lambda}
    For each $N(\lambda)$ the estimate $\hat{\lambda}$ is between:
    \begin{equation*}
        1 +  \frac{ \E[\hat{r}_{i,b_j}] - \E[\hat{r}_{i,j}] }{\E[\hat{r}_{i,b_i}] - \E[\hat{r}_{i,j}] } -\Omega(\delta) \leq \hat{\lambda}  \leq 1 +  \frac{  \E[\hat{r}_{i,b_j}] - \E[\hat{r}_{i,j}]}{\E[\hat{r}_{i,b_i}] - \E[\hat{r}_{i,j}] } + O( \delta) ,
    \end{equation*}
    with probability at least $4 \exp (-2M\delta^2)$ under the assumption of $ (A) : | \E[\hat{r}_{i,b_j}] - \E[\hat{r}_{i,j}] | \geq o( \delta)$ and  $| \E[\hat{r}_{i,i}] - \E[\hat{r}_{i,j}] | - 2 \delta \geq o( \delta)$. Where $\hat{r}_{i,i}$ is the estimator that algorithm builds of random variable $ R_{i,i} \sim \Bern \left(r_i \cdot \left( (1 - \lambda)^{\tilde N(\lambda)} (q_{0,i} - b_i ) + b_i \right) \right)  $, $q_{0,i}$ is the value of state before we start sampling. Respectively, $\hat{r}_{i,b_j}$ is the estimator that algorithm builds of random variable $ R_{i,b_j} \sim \Bern \left(r_i \cdot \left( (1 - \lambda)^{\tilde N(\lambda)} (q_{0,i} - b_j ) + b_j \right) \right)  $, $q_{0,b_j}$ is the value of state before we start sampling. And $\hat{r}_{i,j}$ is the estimator that algorithm builds of random variable $ R_{i,j} \sim \Bern \left(r_i \cdot \left( (1 - \lambda)^{\tilde N(\lambda)} (q_{0,j} - \frac{ b_j + (1 - \lambda) b_i }{2 - \lambda} ) + \frac{ b_j + (1 - \lambda) b_i }{2 - \lambda}  \right) \right)  $ where $q_{0,j}$ is the value of state before we start sampling.
\end{lemma}
\begin{proof}
    For the estimator $\hat{r}_{i,b_j}$ we begin from a random state $q_{0,b_j}$ and then the we play arm $j$ for $\tilde N(\lambda)$ rounds so the state becomes from Lemma~\ref{lem:state-approx-bi}:
    $$ q_{t} = (1 - \lambda)^{\tilde N(\lambda)} (q_{0,b_j} - b_j) + b_j.$$
    Using Hoeffding's Inequality as in Lemma~\ref{lem:hv-estimator} we get:
    $$ \Pr \left[ \left| \hat{r}_{i,b_j} - r_i \cdot \left( (1 - \lambda)^{\tilde N(\lambda)} (q_{0,b_j} - b_j) + b_j \right) \right| \geq \delta \right] \leq 2 \exp \left( -2 M \delta^2 \right) $$
    Then, for $\hat{r}_{i,i}$ we fist play arm $j$ for $\tilde N(\lambda)$ rounds and the state equals to another random value $q_{0,i}$. Then, we play arm $i$ for $\tilde N(\lambda)$ rounds and the state becomes $ q_t = (1- \lambda)^{\tilde N(\lambda)} (q_{0,i} - b_i) + b_i$ and then we take a sample from $\Bern(r_i q_t) $. Thus, the expected value of $\E [\hat{r}_{i,i}] = r_i \cdot \left( (1- \lambda)^{\tilde N(\lambda)} (q_r - b_i) + b_i \right)$, and using Hoeffding's inequality we get:
    $$ \Pr \left[ \left| \hat{r}_{i,j} - r_i \cdot \left( (1 - \lambda)^{\tilde N(\lambda)} (q_{0,i} - b_i) + b_j \right) \right| \geq \delta \right] \leq 2 \exp \left( -2 M \delta^2 \right) $$
    Then the same applies for $\hat{r}_{i,j}$:
     $$ \Pr \left[ \left| \hat{r}_{i,j} - r_i \cdot \left( (1 - \lambda)^{\tilde N(\lambda)} \left(q_{0,j} - \frac{b_j + (1 - \lambda)b_i}{2 - \lambda} \right) + \frac{b_j + (1 - \lambda)b_i}{2 - \lambda} \right) \right| \geq \delta \right] \leq 2 \exp \left( -2 M \delta^2 \right) $$
     Thus, $\hat{\lambda}$ is upper bounded by:
     \begin{align*}
         \hat{\lambda} &=  \frac{\hat{r}_{i,b_i} + \hat{r}_{i,b_j} - 2 \hat{r}_{i,j}}{\hat{r}_{i,b_i} - \hat{r}_{i,j}} \\
         &= 1 - \left| \frac{ \hat{r}_{i,b_j} - \hat{r}_{i,j}}{\hat{r}_{i,b_i} - \hat{r}_{i,j}} \right| \leq 1 -  \frac{ | \E[\hat{r}_{i,b_j}] - \E[\hat{r}_{i,j}] | - 2 \delta}{|\E[\hat{r}_{i,b_i}] - \E[\hat{r}_{i,j}]| + 2\delta} &\tag{ $\frac{ \hat{r}_{i,b_j} - \hat{r}_{i,j}}{\hat{r}_{i,b_i} - \hat{r}_{i,j}} \leq 0$ as ($\lambda \leq 1$)}\\
         &\leq  1 -  \frac{ | \E[\hat{r}_{i,b_j}] - \E[\hat{r}_{i,j}] |}{|\E[\hat{r}_{i,b_i}] - \E[\hat{r}_{i,j}]| } + O(\delta) &\tag{From assumption (A)} \\
         &= 1 +  \frac{  \E[\hat{r}_{i,b_j}] - \E[\hat{r}_{i,j}] }{\E[\hat{r}_{i,b_i}] - \E[\hat{r}_{i,j}] } + O(\delta) 
     \end{align*}
     and lower bounded by
     \begin{align*}
         \hat{\lambda} &=  \frac{\hat{r}_{i,b_i} + \hat{r}_{i,b_j} - 2 \hat{r}_{i,j}}{\hat{r}_{i,b_i} - \hat{r}_{i,j}} \\
         &= 1 - \left| \frac{ \hat{r}_{i,b_j} - \hat{r}_{i,j}}{\hat{r}_{i,b_i} - \hat{r}_{i,j}} \right| \geq 1 -  \frac{ | \E[\hat{r}_{i,b_j}] - \E[\hat{r}_{i,j}] | + 2 \delta}{|\E[\hat{r}_{i,b_i}] - \E[\hat{r}_{i,j}]| - 2\delta} &\tag{ $\frac{ \hat{r}_{i,b_j} - \hat{r}_{i,j}}{\hat{r}_{i,b_i} - \hat{r}_{i,j}} \leq 0$ as ($\lambda \leq 1$)}\\
         &= 1 -  \frac{ | \E[\hat{r}_{i,b_j}] - \E[\hat{r}_{i,j}] | }{|\E[\hat{r}_{i,b_i}] - \E[\hat{r}_{i,j}]| } - \Omega(\delta) &\tag{From assumption $(A)$} \\
         &= 1 +  \frac{ \E[\hat{r}_{i,b_j}] - \E[\hat{r}_{i,j}]  }{\E[\hat{r}_{i,b_i}] - \E[\hat{r}_{i,j}] } - \Omega(\delta)
     \end{align*}
     These are satisfied with probability at least $4 \exp (-2 M \delta)$.
    \end{proof}

Now we move on the main Algorithm~\ref{algo:unknown-lambda}. First, we take some samples to distinguish if $\lambda$ is near $0$ or not and if the $r_i$'s and the gaps $b_i - b_j$ are small.
If it is near $0$ we call $EXP3.P$ otherwise, we try $\tilde N(\lambda)$s in the Algorithm~\ref{algo:gen-unknown-lambda} until we find the right one.

\begin{algorithm}[htbp]
\caption{$\BDES$ Unknown $\lambda$}\label{algo:unknown-lambda}
\DontPrintSemicolon
\LinesNumbered
\SetAlgoNoEnd
\tcc{Separate to small and big $\lambda$}
\For{each arm $i \in [K]$}{
Play arm $i$ for $(T/K)^{2/3}$ rounds. \tcp{arm $i$ chosen at random}
\tcc{State now is at: $q_t \approx b_i$, \emph{if $\lambda$ is big enough}}
\For{$(T/K)^{2/3}$ rounds}{
    Play arm $i$ (same as in Line 1) and update estimator $\hat{\mu}_i \gets  \hat{r}_i + R_i/(T/K)^{2/3}$. \tcp{$rew_t \sim \Bern(q_t  r_i)$}
} 
}
Choose a random arm $j \in [K]$. \;
\For{each $i \in [K] \setminus j$}{
Play arms $i,j$ alternately for $2 (T/K)^{2/3}$ rounds.\;
\For{$2(T/K)^{2/3}$ rounds}{
    Play arm $i$ and update $\hat{\mu}_{(i,j)} \gets  \hat{\mu}_{(i,j)} + R_{i,j}/ \sqrt{T}$.\;
    Play arm $j$.\;
}
}
\tcc{Small $\lambda$.}
\If {$ | \hat{\mu}_i - \hat{\mu}_{i,j}|  \leq 3 \frac{\sqrt{\log T}}{(T/K)^{1/3}}$}{
    Call \textrm{EXP3.P} \;
}
\tcc{Big $\lambda$.}
\Else{
    $\tilde N (\lambda) \gets \log T$\;
     Call Algorithm~\ref{algo:gen-unknown-lambda} for $\tilde N (\lambda)$, arms $i,j$ and get $\hat{\lambda}_1$ \;
     Call Algorithm~\ref{algo:gen-unknown-lambda} for $2\tilde N (\lambda)$, arms $i,j$ and get $\hat{\lambda}_2$ \;
     $\tilde N (\lambda) \gets 4 \tilde N (\lambda) $ \;
    \While{$ | \hat{\lambda}_1 - \hat{\lambda}_2| \geq \delta $}{
        Call Algorithm~\ref{algo:gen-unknown-lambda} for $\tilde N (\lambda)$, arms $i,j$ and get $\hat{\lambda}_3$ \;
        $\hat{\lambda}_1 \gets \hat{\lambda}_2 $ \;
        $\hat{\lambda}_2 \gets \hat{\lambda}_3$ \;
        $\tilde N(\lambda) \gets 2\tilde N (\lambda)$
    }
    Call Algorithm~\ref{algo:ETC-known-iR} for $\lambda = \hat{\lambda}_1$
}
\end{algorithm}

In lemmas~\ref{lem:estimator-ii} and lemmas~\ref{lem:estimator-ij} we bound the estimators $\hat{\mu}_i$ and $\hat{\mu}_{i,j}$. Then, we bound their difference and prove that if real $\lambda$ is small we call $EXP3.P$ otherwise we are trying to learn real $\lambda$.
 
\begin{lemma}\label{lem:estimator-ii}
    For the estimators $\hat{\mu}_i$ it holds:
    $$ \Pr \left[|\hat{\mu}_i - Y_i| \geq \frac{\sqrt{\log T}}{(T/K)^{1/3}} \right] \leq \frac{1}{T^2}, $$
    where  $$Y_i =r_i \left( (1 - \lambda)^{(T/K)^{2/3} } \frac{ 1 - (1 - \lambda)^{(T/K)^{2/3} }}{T^{2/3} \lambda} (q_{0,i} - b_i) + b_i \right) $$ and $q_{0,i}$ the value of state before we start building the estimator.
\end{lemma}
\begin{proof}
    Let $q_{0,i}$ be the state before we start pulling arm $i$. Then, after $(T/K)^{2/3}$ pulls of this arm the state becomes:
    \begin{align*} q_{(T/K)^{2/3} + 1, i} &= (1 - \lambda)^{(T/K)^{2/3}} + \lambda \sum_{s = 0}^{T^{2/3}-1} (1 - \lambda)^{s} b_i \\
    &= (1 - \lambda)^{(T/K)^{2/3}} + \lambda \cdot \frac{1 - (1- \lambda)^{(T/K)^{2/3}}}{1 - (1 -\lambda)} b_i \\
    &= (1 - \lambda)^{(T/K)^{2/3}} (q_{0,i} - b_i) + b_i
    \end{align*}
    Then, for another $(T/K)^{2/3}$ rounds the algorithm samples reward from $\Bern ( r_i \cdot q_{(T/K)^{2/3} + s,i })$ for $s \in [(T/K)^{2/3}]$, where for each $s \in [(T/K)^{2/3}]$ it applies:
    \begin{align*}
        q_{(T/K)^{2/3} + s,i} &= (1 - \lambda)^{(T/K)^{2/3} + s} \cdot (q_{s,i} - b_i) + b_i 
    \end{align*}
    Thus, the expected reward for each $s \in [(T/K)^{2/3}]$ is $ \E [R_{i,s}] = r_i \cdot \left((1 - \lambda)^{(T/K)^{2/3} + s} \cdot (q_{0,i} - b_i) + b_i \right) $, and
    \begin{align*}
        \E \left[ \sum_{s =0}^{(T/K)^{2/3} - 1} r_i q_t\right] &= \sum_{s =0}^{(T/K)^{2/3} - 1} r_i \cdot \left((1 - \lambda)^{(T/K)^{2/3} + s} \cdot (q_{0,i} - b_i) + b_i \right) \\
        &=r_i \left( (1 - \lambda)^{(T/K)^{2/3} } \frac{ 1 - (1 - \lambda)^{(T/K)^{2/3} }}{\lambda} (q_{0,i} - b_i) + (T/K)^{2/3} b_i \right)
    \end{align*} 
    Using Hoeffding's inequality on $\hat{\mu}_i$ we get the result.
    
\end{proof}

\begin{lemma} \label{lem:estimator-ij}
    For the estimators $\hat{\mu}_{i,j}$ it applies:
    $$ \Pr \left[|\hat{\mu}_{i,j} - Y_{i,j}| \geq \frac{\sqrt{\log T}}{(T/K)^{1/3}} \right] \leq \frac{1}{T^2}, $$
    where  $$Y_{i,j} =\frac{(1 - \lambda)^{2(T/K)^{2/3}}}{(T/K)^{2/3}}\left( r_i q_{0,(i,j)} - \frac{r_i b_j}{2 - \lambda} - (1- \lambda) \frac{r_i b_i}{2 - \lambda}\right) \frac{1 - (1 -\lambda)^{2(T/K)^{2/3}}}{1 - (1 - \lambda)^2} +  \frac{r_i b_j + (1- \lambda) r_i b_i}{2 - \lambda}, $$
    and  $q_{0, (i,j)}$ the value of state before we start building the estimator.
\end{lemma}

\begin{proof}
After alternating between a random arm $i$ with arm $j$ for $(T/K)^{2/3}$ the state becomes (lemma~\ref{lem:alt-2-arms}):
$$ q_{2(T/K)^{2/3} + 1,(i,j)} = (1 - \lambda)^{2(T/K)^{2/3}} q_{0,(i,j)} + \frac{1 - (1 - \lambda)^{2(T/K)^{2/3}}}{2 - \lambda} b_j + (1 - \lambda)\frac{1 - (1- \lambda)^{2(T/K)^{2/3}}}{2 - \lambda} b_i. $$
And after that for another $(T/K)^{2/3}$ the algorithm samples reward from $ \Bern (r_i \cdot q_{(T/K)^{2/3} + 2s + 1}  )$, where 
\begin{align*}
     q_{2(T/K)^{2/3} + 2s + 1,(i,j)} = (1 - \lambda)^{2(T/K)^{2/3} +2s} q_{0,(i,j)} + \frac{1 - (1 - \lambda)^{2(T/K)^{2/3}+2s}}{2 - \lambda} b_j
    + (1 - \lambda)\frac{1 - (1- \lambda)^{2(T/K)^{2/3} + 2s}}{2 - \lambda} b_i.
\end{align*}
for $s = \{0,...,(T/K)^{2/3}-1\}$. Thus, the expected reward $\E[R_{(i,j),s}] = r_i \cdot  (1 - \lambda)^{2(T/K)^{2/3} +2s} q_{0,(i,j)}$ $ + \frac{1 - (1 - \lambda)^{2(T/K)^{2/3}+2s}}{2 - \lambda} r_i b_j + (1 - \lambda)\frac{1 - (1- \lambda)^{2(T/K)^{2/3} + 2s}}{2 - \lambda} r_i b_i $. Thus,
\begin{align*}
    \E [\sum_{s = 0}^{(T/K)^{2/3}-1} R_{(i,j),s}] &= \sum_{s = 0}^{(T/K)^{2/3}-1}  r_i \cdot  (1 - \lambda)^{2(T/K)^{2/3} +2s} q_{0,(i,j)} \\
    & + \frac{1 - (1 - \lambda)^{2(T/K)^{2/3}+2s}}{2 - \lambda} r_i b_j + (1 - \lambda)\frac{1 - (1- \lambda)^{2(T/K)^{2/3} + 2s}}{2 - \lambda} r_i b_i  \\
    &= (1 - \lambda)^{2(T/K)^{2/3}}\left( r_i q_{0,(i,j)} + \frac{r_i b_j}{2 - \lambda} + (1- \lambda) \frac{r_i b_i}{2 - \lambda}\right) \frac{1 - (1 -\lambda)^{2(T/K)^{2/3}}}{1 - (1 - \lambda)^2} \\
    & + (T/K)^{2/3} \frac{r_i b_j + (1- \lambda) r_i b_i}{2 - \lambda}.
\end{align*}
Thus, using Hoeffding's inequality on $\hat{\mu}_{i,j}$ we get the result.
\end{proof}
If we take a union bound on all estimators we take that for all estimators $\hat {r}_i$ and $\hat{r}_{i,j}$ apply lemma~\ref{lem:estimator-ii} and lemma~\ref{lem:estimator-ij} with probability at least $2/T$ as $K < T$. In the continuation of the analysis we assume that applies.

\begin{lemma} 
    If $\lambda$ is in $[0, \tilde \Theta(1/T)]$ then the algorithm achieves regret:
    \begin{equation*}
        \regDES (T) = 
        \begin{cases}
          \tcalO \left( K^{1/3}  T^{2/3} \right) & \textrm{for } \lambda \in [0, \Theta(1/T^2)]  \\
          \calO \left( T^{b/a} \right) & \textrm{for } \lambda = T^{-a/b} \\
          (1- 1/e) \OPT & \textrm{for } \lambda = \Theta(1/T).
        \end{cases}
    \end{equation*}
\end{lemma}
\begin{proof}
    If $ \lambda$ is in $[0, \tilde \Theta(1/T)]$
    the states $q_{t}$ during the $K^{1/3}T^{2/3}$ rounds of building the estimators $\hat{\mu}_i$ and $\hat{\mu}_{i,j}$:
    $$ q_{t} \geq (1 - \lambda)^{K^{1/3}T^{2/3}} q_0 \geq (1 - 1/T)^{K^{1/3}T^{2/3}} > 1 - \frac{1}{T^{1/3}} $$
    Thus, $| Y_{i} - Y_{(i,j)}|  \leq r_i|q_t - q_t' |$, with $t \in [0, K^{1/3}T^{2/3}]$ thus, $|q_t - q_t'| \leq \frac{1}{T^{2/3}}$ and so $ | Y_{i} - Y_{(i,j)}|  \leq \frac{1}{T^{1/3}}$ for all arms $i,j$ in $[K]$. This means that $| \mu_i - \mu_{i,j}| \leq \frac{\sqrt{\log T}}{(T/K)^{1/3}}$ and  after the $K \cdot (T/K)^{2/3} = K^{1/3} T^{2/3}$ rounds we run $EXP3.P$. From Section~\ref{sec:small-lambda} we get :
    $$ \Regret_{\DES} (T) = 2K^{1/3} T^{2/3} + \calO (\sqrt{KT \log T}) + (1 - (1 - \lambda)^T) \OPT $$
    which ends the proof.
\end{proof}

\begin{lemma}
    If $\max_{i,j \in [K]} \{ |b_i - b_j| \} \leq \sqrt{\log T}/(T/K)^{1/3}$ then $\forall \lambda \in [0,1]$ Algorithm~\ref{algo:unknown-lambda} after $K^{1/3} T^{2/3}$ calls $EXP3.P$ and achieves $R_{\DES} (T) = \calO \left( K^{1/3} T^{2/3}  \right)$ regret.
\end{lemma}
\begin{proof}
    We first prove prove $|q_t \left(H_{1:t-1}\right) - q_t \left(H_{1:t-1} ' \right)| \leq \sqrt{ \log T}/(T/K)^{1/3}$. If $\max_{i,j \in [K]} \{| b_i - b_j| \} \leq \sqrt{\log T}/(T/K)^{1/3}$ then for all $i,j$ it applies $|b_i - b_j| \leq \sqrt{\log T}/(T/K)^{1/3}$.
    \begin{align*}
        |q_t \left(H_{1:t-1}\right) - q_t \left(H_{1:t-1} ' \right)| &=  (1-\lambda)^t + \lambda \cdot \sum_{s = 0}^{t-1} (1 - \lambda)^{t-1-s} \cdot b_{I_s} -  (1-\lambda)^t + \lambda \cdot \sum_{s = 0}^{t-1} (1 - \lambda)^{t-1-s} \cdot b_{I_s}' \\
        & = \lambda \sum_{s = 0}^{t-1} (1 - \lambda)^{t-1-s} \cdot b_{I_s} - \lambda \sum_{s = 0}^{t-1} (1 - \lambda)^{t-1-s} \cdot b_{I_s}'  = \lambda \sum_{s = 0}^{t-1} (1 - \lambda)^{t-1-s} \cdot (b_{I_s} - b_{I_s}') \\
        &\leq \left(1 - (1 - \lambda)^{t} \right) (b_{I_s} - b_{I_s}') \leq  \sqrt{\log T}/(T/K)^{1/3}. \numberthis{\label{eq:diff-states}}
    \end{align*}
    Since, $|b_i - b_j| \leq \sqrt{\log T}/(T/K)^{1/3}$ for all $i,j$ and ~\ref{eq:diff-states} we get $|Y_i - Y_{i,j}| \leq \frac{\sqrt{\log T}}{(T/K)^{1/3}}$ and so on $|\hat{\mu}_{i} - \hat{\mu}_{i,j}|\leq \frac{3\sqrt{\log T}}{(T/K)^{1/3}}$.
    Thus, the algorithm will call $EXP3.P$ and the regret will be:
    \begin{align*}
        \Regret_{\DES}(T) &= \E \left[\sum_{t \in [T]} r_{\pi_t^{\star}} \cdot q_t \left( H_{1:t-1}^{\pi^{\star}} \right) - \sum_{t \in [T]} r_{I_t} \cdot q_t \left(H_{1:t-1}^\ALG\right) \right] + 2K^1/3 T^{2/3}\\
        &= \E \left[\sum_{t \in [T]} r_{\pi_t^{\star}} \cdot q_t \left( H_{1:t-1}^{\pi^{\star}} \right) - r_{i^{\star}} q_t \left( H_{1:t-1}^{\ist} \right) - \sum_{t \in [T]} r_{I_t}  \cdot q_t \left(H_{1:t-1}^\ALG\right) - r_{\ist} q_t \left( H_{1:t-1}^{\ist} \right) \right]  + 2K^{1/3} T^{2/3} \\
        &\leq \sum_{t \in [T]} r_{\pi_t^{\star}} \cdot q_t \left( H_{1:t-1}^{\pi^{\star}} \right) - r_{i^{\star}} q_t \left( H_{1:t-1}^{\ist} \right) + \calO (\sqrt{KT \log T}) + 2K^{1/3} T^{2/3} \\
        &\leq  \sum_{t \in [T]} r_{i^{\star}} \left( \cdot q_t \left( H_{1:t-1}^{\pi^{\star}} \right) - q_t \left( H_{1:t-1}^{\ist} \right) \right) + \calO (\sqrt{KT \log T}) + 2K^{1/3} T^{2/3} \\
        &\leq  \sum_{t \in [T]} r_{i^{\star}} \frac{K^{1/3}}{T^{1/3}} + \calO (\sqrt{KT \log T}) + 2K^{1/3} T^{2/3} \leq \calO (K^{1/3} T^{2/3})
    \end{align*}
    where $r_{\ist}$ is the $r_i$ of the best fixed arm.
\end{proof}

\begin{theorem}\label{thm:unknown-lambda-gen}
    If $\lambda > O(K^{1/3}/ T^{1/3})$ and $\max_{i,j \in [K]} \{r_i ( b_i - b_j )\} > \omega (\sqrt{\log T}/(T/K)^{2/3})$ then Algorithm~\ref{algo:unknown-lambda} calls Algorithm~\ref{algo:gen-unknown-lambda} and achieves regret $\regDES (T) =  \calO\left( \left(\frac{K \log (T) \log (\lambda)}{\log (1 - \lambda)}\right)^{1/3} \cdot T^{2/3}\right)$
\end{theorem}
\begin{proof}[Proof of \cref{thm:unknown-lambda-gen}]
    First we prove that exist $i,j$ $| Y_{i} - Y_{i,j} | \geq \Omega ( \sqrt{\log T} K^{1/3}/T^{1/3})$ and so the algorithm will not run $EXP3.P$. If $\lambda > K^{1/3}/T^{1/3}$  the $1/\lambda T^{2/3} < K^{1/3}/ T^{1/3}$ and so on for all arms $i,j$ in $[K]$:
    $$ \left |(1 - \lambda)^{(T/K)^{2/3}} \frac{1 - (1 - \lambda)^{(T/K)^{2/3}}  (q_{s,i} - b_i)}{\lambda T^{2/3}}\right| < \frac{K^{1/3}}{T^{1/3}}   $$
    and 
    $$ \left | \frac{(1 - \lambda)^{2(T/K)^{2/3}}}{(T/K)^{2/3}}\left( r_i q_{s,(i,j)} - \frac{r_i b_j}{2 - \lambda} - (1- \lambda) \frac{r_i b_i}{2 - \lambda}\right) \frac{1 - (1 -\lambda)^{2(T/K)^{2/3}}}{1 - (1 - \lambda)^2} \right| < \frac{ K^{1/3} }{T^{1/3}}   $$
    Thus, $Y_i - Y_{(i,j)} \geq \left| r_i b_i - r_i \frac{b_j + (1- \lambda)b_i}{2- \lambda} \right| - 2 \frac{ K^{1/3} }{T^{1/3}}$. Let $\ist, \jst$ be two arms that satisfy $(r_{\ist} b_{\ist} - b_{\jst})\geq \frac{\sqrt{\log T} K^{1/3}}{T^{1/3}}$ then for the random arm $j$ that we choose in line $5$ of Algorithm~\ref{algo:unknown-lambda} applies $\max\{ |b_{\ist} - b_j|, | b_{\jst} - b_j |\} \geq   O(\sqrt{\log T}/(T/K)^{2/3})/2 = O(\sqrt{\log T}/(T/K)^{2/3})$
    Thus, for that pair:
    \begin{align*}
        | Y_{i} - Y_{i,j} | & \geq r_i \left| b_i - \frac{ b_j + (1- \lambda) b_i}{2 - \lambda} \right| - \frac{2}{T^{1/3}} \\
        &= r_i \left| \frac{b_i - b_j}{2 - \lambda}\right| - \frac{2}{T^{1/3}} \geq \Omega( \sqrt{\log T}/ (T/K)^{1/3})
    \end{align*} 
    Thus, the Algorithm~\ref{algo:unknown-lambda} goes to \emph{else}.
    Then, we call Algorithm~\ref{algo:gen-unknown-lambda} for $\tilde N (\lambda) = \log T$ and for $\tilde N (\lambda) = 2 \log T$ and we keep calling it by doubling the $N(\lambda)$s until $|\hat{\lambda}_1 - \hat{\lambda}_2| \leq \delta $. Assume, $\hat{\lambda}_1$ and $\hat{\lambda}_2$ be the $\lambda$'s when event $|\hat{\lambda}_1 - \hat{\lambda}_2| \leq \delta $ occurs. 
    Now we use lemma~\ref{lem:hat-lambda} to bound $| \hat{\lambda}_1 - \hat{\lambda}_2|$, also observe that the assumptions we made in the statement of this lemma applies as $\lambda > O(K^{1/3}/T^{1/3})$.(For ease of the analysis assume $A = r_i (1 - \lambda)^{\tilde N_1(\lambda)} (q_{0,i} - b_i) + r_i b_i - r_i (1 - \lambda)^{ \tilde N_1(\lambda)} (q_{0,j} - b_{i,j}) - r_i \frac{b_j + (1 - \lambda) b_i}{2 - \lambda}$
    and $B = r_i (1 - \lambda)^{\tilde N_2(\lambda)} (q_{0,i} - b_i) + r_i b_i - r_i (1 - \lambda)^{\tilde N_2(\lambda)} (q_{0,j} - b_{i,j}) - r_i \frac{b_j + (1 - \lambda) b_i}{2 - \lambda}$, where $b_{i,j} = \frac{b_j + (1 - \lambda) b_i}{2 - \lambda} $.)
    \begin{align*}
        |\hat{\lambda}_1 - \hat{\lambda}_2| &= \left| 1 -  \frac{ \hat{r}_{i,b_j}^1 - \hat{r}_{i,j}^1}{\hat{r}_{i,b_i}^1 - \hat{r}_{i,j}^1}  - 1 +  \frac{ \hat{r}_{i,b_j}^2 + \hat{r}_{i,j}^2}{\hat{r}_{i,b_i}^2 - \hat{r}_{i,j}^2} \right| \\
        &\geq \left|  \frac{ \E[\hat{r}_{i,b_j}^1] - \E[\hat{r}_{i,j}^1]}
        {\E[\hat{r}_{i,b_i}^1] - \E[\hat{r}_{i,j}^1] }   - \frac{ \E[\hat{r}_{i,b_j}^2] - \E[\hat{r}_{i,j}^2] }{\E[\hat{r}_{i,b_i}^2] - \E[\hat{r}_{i,j}^2] } \right| - \Omega(\delta)\\
        &\geq \bigg| \frac{ r_i (1 - \lambda)^{\tilde N_1(\lambda)} (q_{0,b_j} - b_j) + r_i b_j - r_i (1 - \lambda)^{\tilde N_1(\lambda)} (q_{0,j} - b_{i,j}) - r_i \frac{b_j + (1 - \lambda) b_i}{2 - \lambda} }{r_i (1 - \lambda)^{\tilde N_1(\lambda)} (q_{0,i} - b_i) + r_i b_i - r_i (1 - \lambda)^{ \tilde N_1(\lambda)} (q_{0,j} - b_{i,j}) - r_i \frac{b_j + (1 - \lambda) b_i}{2 - \lambda} }  \\
        &  - \frac{ r_i (1 - \lambda)^{\tilde N_2(\lambda)} (q_{0,b_j} - b_j) + r_i b_j - r_i (1 - \lambda)^{\tilde N_2(\lambda)} (q_{0,j} - b_{i,j}) - r_i \frac{b_j + (1 - \lambda) b_i}{2 - \lambda} }{r_i (1 - \lambda)^{\tilde N_2(\lambda)} (q_{0,i} - b_i) + r_i b_i - r_i (1 - \lambda)^{\tilde N_2(\lambda)} (q_{0,j} - b_{i,j}) - r_i \frac{b_j + (1 - \lambda) b_i}{2 - \lambda} } \bigg| - \Omega(\delta)\\
        &\geq  \bigg| \frac{ r_i (1 - \lambda)^{\tilde N_{1} (\lambda ) }(q_{0,b_j} - b_j) - r_i (1 - \lambda)^{\tilde N_{1} (\lambda ) }(q_{0,j} - b_{i,j})  }{\max(A,B) } \\
        & - \frac{ r_i (1 - \lambda)^{\tilde 2N_{1} (\lambda ) }(q_{0,b_j} - b_j) - r_i (1 - \lambda)^{2\tilde N_{1} (\lambda ) }(q_{0,j} - b_{i,j})  }{\max\{A,B\}}\bigg| \\
        &\geq  \left|\frac{(1 - \lambda)^{\tilde N_1 (\lambda)} \left( r_i(q_{0,b_j} -b_j)\left(1 - (1- \lambda)^{\tilde N_1(\lambda)}\right)  - r_i (q_{0,j} - b_{i,j}) \left(1 -  (1 - \lambda)^{\tilde N_1 (\lambda)}\right) \right) }{\max \{A,B \}} \right| - \Omega(\delta)
    \end{align*}

    Because $\lambda > \frac{K^{1/3}}{T^{1/3}}$ solving the equation we get $(1 - \lambda)^{N_1(\lambda)} \leq |\Omega (\delta)|$ and also that means that:
    \begin{align*}
        \hat{\lambda}_1 &\leq 1 -  \frac{ | \E[\hat{r}_{i,b_j}] - \E[\hat{r}_{i,j}] | - 2 \delta}{|\E[\hat{r}_{i,b_i}] - \E[\hat{r}_{i,j}]| + 2\delta} \\
        &= 1 + \frac{r_ib_j - r_i \frac{b_j + (1 - \lambda) b_i}{2 - \lambda} }{r_i b_i - r_i\frac{ b_j + (1 - \lambda) b_i}{2 - \lambda} } + O(\delta) \\
        &= 1 - 1 + \lambda + O(\delta) \\
        &= \lambda + O(\delta)
    \end{align*}
    and $\hat{\lambda}_1 \geq \lambda- \Omega(\delta)$ respectively.
    Then we call Algorithm~\ref{algo:ETC-known-iR} for $\lambda = \hat{\lambda}_1$ the Theorem~\ref{thm:regret-ETC} applies and we get an additional $O(\delta) T$ from $\hat{\lambda}_1$. From the exploration of Algorithm~\ref{algo:unknown-lambda} we get an additional $K^{1/3}T^{2/3}$ regret and from Algorithm~\ref{algo:gen-unknown-lambda} an $N(\lambda) \log T$. Thus, the total regret is: 
   \begin{align*} \regDES (T) &=  \calO\left( \left(\frac{K \log (T) \log (\lambda)}{\log (1 - \lambda)}\right)^{1/3} \cdot T^{2/3}\right)+2 K^{1/3} T^{2/3} + O(\delta) \cdot T + N(\lambda)\log T  \\&= \calO\left( \left(\frac{K \log (T) \log (\lambda)}{\log (1 - \lambda)}\right)^{1/3} \cdot T^{2/3}\right) \end{align*}
\end{proof}

\end{document}